\documentclass[twoside]{article}

\usepackage[accepted]{aistats2022arxiv}
\usepackage[round]{natbib}

\usepackage{mathtools} 
\usepackage{booktabs} 
\usepackage{tikz} 
\usepackage[utf8]{inputenc} 
\usepackage[T1]{fontenc}    
\usepackage{hyperref}       
\usepackage{url}            
\usepackage{booktabs}       
\usepackage{amsfonts}       
\usepackage{nicefrac}       
\usepackage{microtype}      
\usepackage{amsmath}
\usepackage{subcaption}
\usepackage{enumitem}
\usepackage{commath} 

\usepackage{caption}
\usepackage{algorithm}
\usepackage[noend]{algpseudocode}
\usepackage{wrapfig}
\usepackage{enumitem}
\usepackage{xr}

\usepackage{amsthm}
\usepackage{thmtools}
\usepackage{thm-restate}

\newtheorem{lemma}{Lemma}
\newtheorem{definition}{Definition}

\usepackage[normalem]{ulem}

\usepackage{tikz}
\usetikzlibrary{bayesnet}

\DeclareMathOperator*{\argmax}{argmax}
\DeclareMathOperator*{\Var}{Var}

\DeclareMathOperator*{\Lap}{Lap}

\DeclareMathOperator*{\E}{\mathbb{E}}

\newcommand{\R}{\mathbb{R}}
\newcommand{\F}{\mathcal{F}}
\renewcommand{\P}{\mathbb{P}}

\newcommand{\y}{\mathbf{y}}
\newcommand{\x}{\mathbf{x}}
\renewcommand{\u}{\mathbf{u}}
\newcommand{\w}{w}
\newcommand{\toD}{\rightsquigarrow}

\newcommand{\toP}{\stackrel{P}{\to}}
\newcommand{\N}{\mathcal{N}}

\newcommand{\betapriv}{\hat{\beta}}

\newcommand{\A}{\mathcal{A}}
\newcommand{\B}{\mathcal{B}}

\DeclareMathOperator{\width}{width}

\definecolor{darkspringgreen}{rgb}{0.09, 0.45, 0.27}

\newcommand{\blue}[1]{\textcolor{blue}{ #1 }}
\newcommand{\ds}[1]{\textcolor{red}{\small [#1 \bf--DS]}}
\newcommand{\cf}[1]{\textcolor{blue}{\small [#1 \bf--CF]}}

\usepackage{array}
\usepackage{multicol}

\renewcommand{\ds}[1]{}
\renewcommand{\cf}[1]{}

\usepackage{ifthen}
\newcommand{\showproofs}{1}

\begin{document}

%

%

\twocolumn[

\aistatstitle{Parametric Bootstrap for Differentially Private Confidence Intervals}

\aistatsauthor{
  Cecilia Ferrando, \; Shufan Wang, \; Daniel Sheldon\\
  College of Information and Computer Sciences\\
  University of Massachusetts Amherst\\
  \texttt{\{cferrando\},\{shufanwang\},\{sheldon\}@cs.umass.edu} \\
}

\aistatsaddress{ } ]

%

\author{%
  Cecilia Ferrando, \; Shufan Wang, \; Daniel Sheldon\\
  College of Information and Computer Sciences\\
  University of Massachusetts Amherst\\
  \texttt{\{cferrando\},\{shufanwang\},\{sheldon\}@cs.umass.edu} \\
}

\begin{abstract}
The goal of this paper is to develop a practical and general-purpose approach to construct confidence intervals for differentially private parametric estimation. We find that the parametric bootstrap is a simple and effective solution. It cleanly reasons about variability of both the data sample and the randomized privacy mechanism and applies ``out of the box'' to a wide class of private estimation routines. It can also help correct bias caused by clipping data to limit sensitivity. We prove that the parametric bootstrap gives consistent confidence intervals in two broadly relevant settings, including a novel adaptation to linear regression that avoids accessing the covariate data multiple times. We demonstrate its effectiveness for a variety of estimators, and find that it provides confidence intervals with good coverage even at modest sample sizes and performs better than alternative approaches.
\end{abstract}



\section{Introduction}\label{introduction}

Differential privacy provides a rubric for drawing inferences from data sets without compromising the privacy of individuals. 

This paper is about privately constructing confidence intervals. In the non-private case, approximate methods based on asymptotic normality or the bootstrap~\citep{efron1979bootstrap} apply to a wide range of models and are very widely used in practice.
In the private case, such ``swiss army knife'' methods are hard to find. The situation is complicated by the fact that private estimation procedures are necessarily randomized, which leads to a distinct source of randomness (``privacy noise'') in addition to the usual random draw of a finite sample from a population (``sampling noise'').
We find experimentally that asymptotic methods are significantly less effective in private settings, due to privacy noise that becomes negligible only for for very large sample sizes (Section~\ref{sec:experiments}). Bootstrap approaches face the challenge of incurring privacy costs by accessing the data many times~\citep{brawner2018bootstrap}.

This paper advocates using the parametric bootstrap as a simple and effective method to construct confidence intervals for private statistical estimation. The parametric bootstrap resamples data sets from an estimated parametric model to approximate the distribution of the estimator. It is algorithmically simple, can be used with essentially any private estimator, and cleanly reasons about both sampling noise and privacy noise. Unlike the traditional bootstrap, it is based on post-processing and avoids accessing the data many times, so it often has little or no privacy burden. By reasoning about the distribution of a finite sample, it makes fewer assumptions than purely asymptotic methods and significantly mitigates the problem of non-negligible privacy noise. The parametric bootstrap can also help correct bias in private estimation caused by artificially bounding data to limit sensitivity.

We first introduce the parametric bootstrap and discuss its application to private estimation, including methods to construct confidence intervals and correct bias. We then review parametric bootstrap theory, and apply the parametric bootstrap to obtain provably consistent confidence intervals in two private estimation settings---exponential families and linear regression sufficient statistic perturbation (SSP)---as well as an empirical demonstration for the ``one posterior sample'' (OPS) method~\citep{Wang:2015aa,foulds2016theory,ZZhang:2016aa}. These demonstrate the broad applicability the parametric bootstrap to private estimation.

One limitation of the parametric bootstrap is the restriction to fully parametric estimation. For example, it doesn't apply directly to regression problems that do not have a parametric model for covariates, and may not be appropriate very complex data. 
In our linear regression application, we contribute a novel hybrid bootstrap approach to circumvent this limitation; the resulting method is easy to use and simultaneously estimates regression coefficients and constructs confidence intervals with good coverage properties. A second limitation is computational cost, which scales with the data size. For small or medium data sets, the cost is likely manageable. For very large ones, cheap asymptotic methods will often be adequate (see Section~\ref{sec:experiments}; for exponential families and linear regression with sufficient statistic perturbation, the asymptotic distributions are a relatively simple byproduct of our bootstrap theory). However, it is unknown in general how large data must be for asymptotic methods to perform well.

\section{Background}\label{background}

Differential privacy is a formal definition to capture the notion that, to maintain privacy, the output of an algorithm should remain nearly unchanged if the data of one individual changes. Say that two data sets $X$ and $X'$ of size $n$ are \emph{neighbors} if they differ in exactly one data record.

\begin{definition}[Differential privacy, \citealt{dwork2006calibrating}]
A randomized algorithm $\mathcal{A}$ satisfies $\epsilon$-differential
privacy ($\epsilon$-DP) if, for neighboring data sets $X$ and $X'$, and any subset $O \subseteq \text{Range}(\mathcal{A})$,
$$ \Pr[\mathcal{A}(X) \in O] \leq \exp(\epsilon) \Pr[\mathcal{A}(X') \in O]. $$
\end{definition}

One common way to achieve differential privacy is by injecting calibrated noise onto the statistics computed from the data. Let $f$ be any function that maps data sets to $\mathbb{R}^{d}$. The magnitude of noise required to privatize the computation of $f$ depends on its \emph{sensitivity}. 

\begin{definition}[Sensitivity, \citealt{dwork2006calibrating}]
The sensitivity of a function f is $$\Delta f = \max_{X,X^{'}} \Vert  f(X) - f(X^{'}) \Vert_{1}$$ where $X,X^{'}$ are any two neighboring data sets.
\end{definition}

When $f$ is additive, it is straightforward to bound its sensitivity (proof in Appendix~\ref{sec:proofs}):
\begin{restatable}{claim}{additive}
\label{claim:additive}
Suppose $X = (x_1, \ldots, x_n)$ and $f(X) = \sum_{i=1}^n g(x_i)$ where $g$ maps data points to $\R^m$. Let $\width (g_j) = \max_{x} g_j(x) - \min_x g_j(x)$ where $x$ ranges over the data domain. Then $\Delta f \leq \sum_{j=1}^m \width(g_j)$, which is a constant independent of $n$.
\end{restatable}

Many algorithms satisfy differential privacy by using the Laplace mechanism.
\begin{definition}[Laplace mechanism, \citealt{dwork2006calibrating}]
Given a function $f$ that maps data sets to $\mathbb{R}^{m}$, the Laplace mechanism outputs the random variable $\mathcal{L}(X) \sim \Lap(f(X), \Delta f / \epsilon)$
from the Laplace
distribution, which has density $\Lap(z; u, b) = (2b)^{-m} \exp ( - \left\Vert  z - u \right\Vert_{1} / b)$. This corresponds to
adding zero-mean independent noise $u_i \sim \Lap(0, \Delta f / \epsilon)$ to each component of $f(X)$.
\end{definition}

\section{Parametric Bootstrap}\label{sec:bootstrap}
\begin{wrapfigure}{R}{0.32\textwidth}
\vspace{-23pt}
\begin{minipage}{0.32\textwidth}
\begin{algorithm}[H]
  \caption{Parametric Bootstrap\label{alg:bootstrap}}
  \begin{algorithmic}[1]
  \State \textbf{Input:} $x_{1:n}$, $B$, estimator $\A$
  \State $\hat{\theta}, \hat \tau \leftarrow \A(x_{1:n})$
    \For {$b$ from $1$ to $B$}
        \State $x^*_1, \ldots, x^*_n \sim P_{\hat{\theta}}$
        \State $\hat\theta^{*b}, \hat\tau^{*b} \leftarrow \A(x^*_{1:n})$
    \EndFor
    \State \Return{$\hat{\tau}, \big(\hat\tau^{*1}, \ldots, \hat\tau^{*B}\big)$}
\end{algorithmic}
\end{algorithm}
\end{minipage}
\vspace{-8pt}
\end{wrapfigure}
We consider the standard setup of parametric statistical inference, where a data sample $x_{1:n} = (x_1, \ldots, x_n)$ is observed and each $x_i$ is assumed to be drawn independently from a distribution $P_\theta$ in the family $\{P_\theta: \theta \in \Theta\}$ with unknown $\theta$.

The goal is to estimate some population parameter $\tau = \tau(\theta)$, the \emph{estimation target}, via an estimator $\hat\tau = \hat\tau(x_{1:n})$.\footnote{We use a hat on variables that are functions of the data and therefore random.} We also seek a $1-\alpha$ confidence interval for $\tau$, that is, an interval $[\hat a_n, \hat b_n]$ such that $\P_\theta( \hat a_n \leq \tau \leq \hat b_n) \approx 1-\alpha$, where $\P_\theta$ is the probability measure over the full sample when the true parameter is~$\theta$.
We will require $\hat\tau$ and $[\hat a, \hat b]$ to be differentially private. Our primary focus is not designing private estimators $\hat\tau$, but designing methods to construct private confidence intervals $[\hat a, \hat b]$ that can be used for many estimators and have little additional privacy burden.

The parametric bootstrap is a simple way to approximate the distribution of $\hat\tau$ for confidence intervals and other purposes. It is a variant of Efron's bootstrap~\citep{efron1979bootstrap,efron1981anonparametric, efron1981bnonparametric, efron1986bootstrap}, which runs an estimator many times on simulated data sets whose distribution approximates the original data. In the parametric bootstrap, data sets are simulated from $P_{\hat\theta}$, the parametric distribution with estimated parameter $\hat\theta$.\footnote{In the non-parametric bootstrap, data sets are simulated from the empirical distribution of $x_{1:n}$.} The procedure is shown in Algorithm~\ref{alg:bootstrap}, where $\A$ is an algorithm that computes the estimates $\hat\theta$ and $\hat\tau$ from the data. A simple case is when $\hat\tau(x_{1:n}) = \tau\big(\hat\theta(x_{1:n})\big)$ but in general these may be estimated separately.

The parametric bootstrap is highly compatible with differential privacy. The data is only accessed in Line 2, so the only requirement is that $\A$ be differentially private (which necessitates it is randomized). The remaining steps are post-processing and incur no additional privacy cost. The simulation cleanly handles reasoning about both data variability (Line 4) and randomness in the estimator (Line 5). When the estimation target is $\theta$, we have $\hat\tau = \hat\theta$, and the procedure incurs no privacy cost beyond that of the private estimator $\hat\tau$. In other cases, additional privacy budget is required to estimate the full vector $\theta$ including nuisance parameters; an example is estimating the mean of a Gaussian with unknown variance~\citep{du2020differentially}.

\subsection{Confidence Intervals and Bias Correction}\label{CIsBiasCorrection}

\begin{figure}[h]
  \centering
\begin{minipage}{0.45\textwidth}
\vspace{-20pt}
\hspace{-15pt}
\begin{table}[H]
  \scriptsize
  \setlength{\tabcolsep}{2.8pt}
  \renewcommand{\arraystretch}{1.2}
  \newcolumntype{L}{>{\raggedright\arraybackslash}m{3cm}}
  \caption{ Bootstrap confidence intervals. $\hat{\xi}_{\gamma}$ is the $1-\gamma$ quantile of $\hat\tau^* - \hat\tau$ or $(\hat\tau^* - \hat\tau)/\hat\sigma^*$, and $\hat{\zeta}_\gamma$ is the $1-\gamma$ quantile of $\hat\tau^*$. \label{table:CI-methods}}
\begin{tabular}{@{}lccc@{}}
  \toprule
  Interval       & Target & Target interval & $\tau$ interval \\
  \midrule
  Pivotal  &
  $\hat\tau - \tau$ &
  $[\hat\xi_{1-\frac{\alpha}{2}}, \hat\xi_{\frac{\alpha}{2}}]$ &
  $[\hat\tau - \hat\xi_{\frac{\alpha}{2}},\ \hat\tau - \hat\xi_{1-\frac{\alpha}{2}}]$ \\
  \midrule
  Studentized\\pivotal &
  $\frac{(\hat\tau - \tau)}{\hat\sigma}$ &
  $[\hat\xi_{1-\frac{\alpha}{2}}, \hat\xi_{\frac{\alpha}{2}}]$ &
  $[\hat\tau - \hat\xi_{\frac{\alpha}{2}} \hat\sigma,\ \hat\tau - \hat\xi_{1-\frac{\alpha}{2}}\hat\sigma]$ 
  \\
  \midrule
  Efron's\\percentile &
  $\hat\tau$ &
  $[\hat\zeta_{1-\frac{\alpha}{2}},\ \hat\zeta_{\frac{\alpha}{2}}]$ &
  $[\hat\zeta_{1-\frac{\alpha}{2}},\ \hat\zeta_{\frac{\alpha}{2}}]$ \\
  %
  \bottomrule
\end{tabular}
\end{table}
\end{minipage}
\end{figure}

There are several well known methods to compute confidence intervals from bootstrap replicates. Three are listed in Table~\ref{table:CI-methods}; note that names are inconsistent in the literature.\footnote{Our mathematical presentation follows~\citet{van2000asymptotic}, but names  follow~\citet{wasserman2006all}. The names ``pivotal'' and ``studentized pivotal'' and are descriptive and avoid the confusion of ``percentile interval'' sometimes referring to the pivotal interval and other times to Efron's percentile interval. The possessive "Efron's"~\citep{van2000asymptotic} clarifies that we use Efron's definition of ``percentile interval''~\citep[e.g.,][]{efron2016computer}.}
The general principle is to treat the pair $(\hat\tau^*, \hat\tau)$ analogously to $(\hat\tau, \tau)$ to approximate the distribution of the latter. The intervals differ according to what function of $(\hat\tau, \tau)$ they target. A simple example is to approximate the ``pivot'' $\hat\tau - \tau$ by $\hat\tau^* - \hat\tau$, which leads to the pivotal interval. To construct it, we estimate the $1-\gamma$ quantile of $\hat\tau^* - \hat\tau$ as the $1-\gamma$ quantile of the bootstrap replicates $(\hat\tau^{*1}-\hat\tau, \ldots, \hat\tau^{*B}-\hat\tau)$. The number of replicates controls the error introduced by this step. This error is usually ignored theoretically because, in principle, it can be reduced arbitrarily with enough computation, and it can be controlled well in practice. The studentized pivotal interval targets $(\hat\tau - \tau)/\hat\sigma$ instead, where $\hat\sigma$ is a standard error estimate of the main procedure; it can converge faster than the pivotal interval~\citep{wasserman2006all}. Efron's percentile interval targets $\hat\tau$ directly and, while simple, its logic is less obvious; it can also be viewed as targeting the pivot $\hat\tau - \tau$ with a ``reversed'' interval, which is how theoretical properties are shown. By approximating $\hat\tau - \tau$ by $\hat\tau^* - \hat\tau$, we can also estimate the bias of $\hat\tau$. This leads to a simple bias corrected estimator $\hat\tau_{\text{bc}}$:
$$
\widehat{\text{bias}} = \mathbb{E}[\hat\tau^* - \hat\tau], \quad \hat{\tau}_{\text{bc}} \leftarrow \hat \tau - \widehat{\text{bias}}.
$$
Similar to the quantiles above, $\mathbb{E}[\hat\tau^* - \hat\tau]$ is estimated as the sample mean over bootstrap replicates.



\subsection{Significance and Connection to Other Resampling Methods for Private Estimation}\label{significance}

The parametric bootstrap can be applied to any parametric estimation problem, a wide range of private estimators, is very accurate in practice, and has little or no additional cost in terms of privacy budget or algorithm development.
These make it an excellent default choice (to our knowledge, the best) for constructing private confidence intervals for any parametric estimation problem with small to medium data sets.

That such a simple and effective choice is available is not articulated in the literature. Two prior works use methods that can be viewed as the parametric bootstrap, but do not discuss the classical procedure and its wide ranging applications, or different facets of bootstrap methodology such as techniques for constructing confidence intervals and bootstrap theory. Specifically, the simulation approach of \citet{du2020differentially} for Gaussian mean confidence intervals is equivalent to the parametric bootstrap with a non-standard variant of Efron's percentile intervals, and performed very well empirically. In their application to independence testing, \citet{gaboardi2016differentially} approximate the distribution of a private test statistic by simulating data from a null model after privately estimating its parameters; this can be viewed as an application of the parametric bootstrap to the null model.

Several other works use resampling techniques that resemble the parametric bootstrap for a similar, but conceptually distinct, purpose~\citep{d2015differential,wang2019differentially,evans2019statistically}. A typical setup is when $\hat\tau = \tau' + \eta$, where $\tau'$ is a non-private estimator and $\eta$ is noise added for privacy. Standard asymptotics are used to approximate $\sqrt{n}(\tau' - \tau)$ as $\N(0, \hat\sigma)$, where $\hat\sigma$ is a (private) standard error estimate for $\tau'$. For the private estimator, this gives $\sqrt{n}(\hat\tau-\tau) \approx \N(0, \hat\sigma) + \sqrt{n}\eta$. Because $\eta$ has known distribution, Monte Carlo sampling can be used to draw samples from $\N(0, \hat\sigma) + \sqrt{n}\eta$ for computing confidence intervals or standard errors. The key distinction is that \emph{standard} asymptotics are used to approximate the distribution of $\tau'$, which captures all variability due to the data, and sampling is used only to combine that distribution with the privacy noise distribution. In contrast, the key feature of a bootstrap method is that it resamples data sets to reason about estimator variability due the random data, and thereby avoids standard asymptotics. This technique also does not apply when the privacy mechanism is more complicated than adding additive noise to a non-private estimate (cf. the OPS example of Sec.~\ref{sec:OPS}).




\section{Bootstrap Theory}\label{theory}

This section gives general results we can use to argue correctness of bootstrap confidence intervals in private settings. We give a general notion of ``bootstrap'' estimator that covers different  resampling methods. Let $(\Omega, \F, \P_\theta)$ be the probability space for $x_1, x_2, \ldots \sim P_\theta$ and $\eta_1, \eta_2, \ldots$ where, for a given $n$, the data is $x_{1:n}$ and $\eta_n$ captures any other randomness used in the privacy mechanism or estimator; we refer to this as the ``outer'' probability space. A bootstrap estimator is defined in terms of a random experiment over an ``inner'' probability space conditional on $\omega \in \Omega$ and $n$. Let $\P_n^*(\cdot | \omega)$ be a Markov kernel defining this space. The traditional bootstrap uses $\P_n^*(\cdot | \omega) = \hat{P}^n$ with $\hat{P}(dx) = \frac{1}{n} \sum_{i} \delta_{x_i}(dx)$; the parametric bootstrap uses $\hat{P} = P_{\hat \theta}$ instead. Our hybrid model in Section~\ref{linearregression} uses a custom resampling method, which gives a custom measure $\P_n^*(\cdot | \omega)$.

For our purposes, a bootstrap estimator of a parameter $\tau(\theta)$ is a random variable $\hat\tau^*_n$ in the inner probability space that simulates the parameter estimate $\hat\tau_n$ of the ``main'' procedure. Typically, the bootstrap estimator arises from running the main procedure on resampled data. That is, if $\hat\tau_n = T_n(\omega)$, then $\hat\tau^* = T_n(\omega^*)$ with $\omega^*  \sim \P_n^*(\cdot \mid \omega)$. Our hybrid OLS bootstrap will deviate slightly from this pattern.

\subsection{Consistency}
Bootstrap ``success'' has to do with the asymptotic distributions of the (approximate) pivot $\sqrt{n}(\hat\tau_n - \tau)$ and its bootstrapped counterpart $\sqrt{n}(\hat\tau^*_n - \hat\tau_n)$. For studentized intervals, the pivot $(\hat\tau_n - \tau)/\hat \sigma_n$ is used instead, where $\hat\sigma_n$ is a standard error estimate of the main procedure; theory for this case is a straightforward extension if $\hat\sigma_n \to \sigma(\theta)$ in $\P_\theta$-probability~\citep{van2000asymptotic,beran1997diagnosing}.

\begin{definition}
  \label{def:consistency}
  The bootstrap estimator $\hat\tau^*_n$ is \emph{consistent} if
\begin{equation}
\begin{split}
\sup_x 
\Big| 
\P_n^*\Big(
   \sqrt{n}(\hat\tau_n^* - \hat\tau_n) \leq t \mid \omega
\Big)
- \\-
\P_\theta\Big(
   \sqrt{n}(\hat\tau_n - \hat\tau) \leq t
   \Big) \Big|
\toP 0 
\end{split}
\end{equation}
with convergence in $\P_\theta$-probability.
\end{definition}
This says that the Kolmogorov-Smirnov distance between the distribution of the pivot and the conditional distribution of the bootstrapped pivot converges to zero, in probability over $\omega$.

In most cases $\sqrt{n}(\hat \tau_n - \tau) \toD T$ for a continuous random variable $T$. In this case it is enough for the bootstrapped pivot to converge to the correct limit distribution.
\begin{lemma}[\citealt{van2000asymptotic}, Eq. (23.2)]
  \label{lem:consistency}
  Suppose $\sqrt{n}(\hat \tau_n - \tau) \toD T$ for a random variable $T$ with continuous distribution function $F$. Then, $\hat\tau^*_n$ is consistent if and only if, for all $t$,
  $$
  \P^*_n \Big(\sqrt{n}(\hat\tau^*_n - \hat\tau_n) \leq t \mid \omega \Big) \overset{P}{\to} F(t).
  $$
\end{lemma}

Consistency is also preserved under continuous mappings: if $\hat\tau^*_n$ is consistent relative to $\sqrt{n}(\hat\tau_n - \tau) \toD T$ and $g$ is continuous, then $g(\hat\tau^*_n)$ is consistent relative to $\sqrt{n}\big(g(\hat\tau_n) - g(\tau)\big)$~\citep{beran1997diagnosing}. In our applications we will show consistency of a bootstrap estimator $\hat\theta^*_n$ for the full parameter vector $\theta$, which implies consistency for continuous functions of $\theta$; a simple application is selecting one entry and constructing a confidence interval.

\subsection{Confidence interval consistency}

Bootstrap consistency implies consistent confidence intervals. The confidence interval $[\hat{a}_n, \hat b_n]$ for $\tau=\tau(\theta)$ is (conservatively) asymptotically consistent at level $1 - \alpha$ if, for all $\theta$,
\begin{equation}
  \liminf_{n \to \infty} \P_\theta \left(\hat a_{n} \leq \tau \leq \hat b_{n}\right) \geq 1 - \alpha.
\end{equation}

\begin{lemma}[\citealt{van2000asymptotic}, Lemma 23.3]
Suppose $\sqrt{n}(\tau_n - \tau) \toD T$ for a random variable $T$ with continuous distribution function and $\tau^*_n$ is consistent. Then the pivotal intervals are consistent, and, if $T$ is symmetrically distributed around zero, then Efron's percentile intervals are consistent. When the analogous conditions hold for the studentized pivot $(\hat\tau_n - \tau)/\hat\sigma_n$, studentized intervals are consistent.
\end{lemma}

\subsection{Parametric bootstrap consistency}

\citet{beran1997diagnosing} showed that asymptotic equivariance of the main estimator guarantees consistency of the parametric bootstrap. Let $H_n(\theta)$ be the distribution of $\sqrt{n}(\hat\tau_n - \tau(\theta))$ under $\P_\theta$. 
\begin{definition}[Asymptotic equivariance,~\citealt{beran1997diagnosing}]\label{definition:AE}
The estimator $\hat\tau_n$ is asymptotically equivariant if $H_n(\theta + h_n/\sqrt{n})$ converges to a limiting distribution $H(\theta)$ for all convergent sequences $h_n$ and all~$\theta$.
\end{definition}
\begin{restatable}[Parametric bootstrap consistency]{theorem}{PBconsistency}\label{PBconsistency}
Suppose $\sqrt{n}(\hat\theta_n - \theta) \toD J(\theta)$ and $\hat\tau_n$ is asympotitcally equivariant with continuous limiting distribution $H(\theta)$. Then the parametric bootstrap estimator $\hat\tau^*_n$ is consistent.
\end{restatable}
All proofs are provided in the appendix.
Furthermore, under reasonably general conditions, the reverse implication is true, with bootstrap failures occurring precisely at those parameter values $\theta_0$ for which asymptotic equivariance does not hold~\citep{beran1997diagnosing}.

\section{Applications}

We apply the parametric bootstrap to three private estimation settings: (i)~exponential families with sufficient statistic perturbation (SSP), (ii)~linear regression with SSP, (iii)~the ``one posterior sample'' (OPS) estimator.

\paragraph{Exponential Families}\label{exponentialfamilies}

A family of distributions is an \emph{exponential family} if $P_\theta$ has a density of the form:
$$
\vspace{-2pt}
p(x; \theta) = h(x) \exp(\theta^T T(x) - A(\theta))
$$
where $h(x)$ is a base measure, $\theta$ is the natural parameter, $T$ is the sufficient statistic function, and $A$ is the log-partition function. Define the log-likelihood function of an exponential family as $\ell(\theta; x) = \log p(x; \theta) - \log h(x) = \theta^T T(x) - A(\theta)$.
The constant term $\log h(x)$ does not affect parameter estimation and is subtracted for convenience. 
For a sample $x_{1:n}$, let $T(x_{1:n}) = \sum_{i=1}^n T(x_i)$. The log-likelihood of the sample is
$$
\ell(\theta; x_{1:n}) = \theta^T T(x_{1:n}) - n A(\theta) := f\big(\theta; T(x_{1:n}) \big),
$$
which depends on the data only through the sufficient statistic $T(x_{1:n})$. The maximum-likelihood estimator (MLE) is
$\hat{\theta} = \argmax_\theta f\big(\theta; T(x_{1:n})\big)$.

A simple way to create a private estimator is sufficient statistic perturbation (SSP); that is, to privatize the sufficient statistics using an elementary privacy mechanism such as the Laplace or Gaussian mechanism prior to solving the MLE problem. 
SSP is a natural choice because $T(x_{1:n})$ is a compact summary of the data and has sensitivity that is easy to analyze, and it often works well in practice~\citep{bernstein2018differentially,foulds2016theory}.
Specifically, it means solving
\begin{equation}
\hat{\theta} = \argmax_\theta f\big(\theta, T(x_{1:n}) + w\big) \tag{SSP-MLE}
\label{SSPMLE}
\end{equation}
where $w$ is a suitable noise vector. This problem has closed form solutions for many exponential families and standard numerical routines apply to others. 
For the Laplace mechanism, $w_j \sim \Lap(\frac{\Delta}{\epsilon})$
\ds{Dropped the factor of 2. BTW, what do we do in our code?}\cf{we use the range of a surrogate sample of size 1000 (the 2x was an old strategy to stay conservative, we dropped in the code long ago)}
for all $j$, where $\Delta = \sum_{j} \width(T_j)$ is an upper bound on the $L_1$ sensitivity of $T(x_{1:n})$ by Claim~\ref{claim:additive}. If $\width(T_j)$ is not known or is unbounded, the analyst must supply bounds and guarantee they are met, e.g., by discarding data points that don't meet the bounds, or clamping them to the bounded interval.

\begin{restatable}{theorem}{MLEconsistency}\label{MLEconsistency}
Let $\hat{\theta}_n$ be the solution to the (SSP-MLE) optimization problem for a sample $x_{1:n}$ from an exponential family model that satisfies the regularity conditions given in \citet[Section 4.4.2]{davison2003statistical}. Then $\sqrt{n}(\hat{\theta}_n - \theta)$ is asymptotically equivariant with limiting distribution $\N(0, I(\theta)^{-1})$, where $I(\theta) = \nabla^2 A(\theta)$ is the Fisher information. This implies consistency of the parametric bootstrap estimator $\hat\theta^*_n$.
\end{restatable}
\ds{Give reference to specific appendix?}

\paragraph{Linear Regression}\label{linearregression}
\label{sec:OLS}

We consider a linear regression model where we are given $n$ pairs\footnote{We use boldface for vectors as needed to distinguish from scalar quantities.}  $(\x_i, y_i)$ with $\x_i \in \R^p$ and $y_i \in \R$ assumed to be generated as
$
y_i = \beta^T \x_i + u_i,
$
where the errors $u_i$ are i.i.d., independent of $\x_i$, zero-mean, and have finite variance $\sigma^2$, and the $\x_i$ are i.i.d. with $\E[\x \x^T] = Q$. We wish to estimate the regression coefficients $\beta \in \R^p$. Let $X \in \R^{n \times p}$ be the matrix with $i$th row equal to $\x_i^T$ and $\y, \u \in \R^N$ be the vectors with $i$th entries $y_i$ and $u_i$, respectively. The \emph{ordinary least squares} (OLS) estimator is:
\begin{equation}
  \hat \beta = (X^{T}X)^{-1} X^{T} \y.
  \label{eq:OLS}
\end{equation}


Like the MLE in exponential families, Eq.~\eqref{eq:OLS} depends on the data only through sufficient statistics $X^T X$ and $X^T \y$, and
SSP is a simple way to privatize the estimator that works very well in practice~\citep{wang2018revisiting}. The privatized estimator is
\begin{equation} \label{eq:privnorm}
    \betapriv = (X^{T}X + V)^{-1} (X^{T} \y + \w),
    \tag{SSP-OLS}
\end{equation}
where $V \in \R^{p \times p}$ and $\w \in \R^p$ are additive noise variables drawn from distributions $P_V$ and $P_w$ to ensure privacy.

For the Laplace mechanism, we use \ds{This paragraph could go to appendix, maybe along with details of $\hat\sigma^2$? We do need to mention $\hat\sigma^2$; the properties we need are that it is private and consistent.}
\begin{align}
V_{jk} &\sim \Lap(0, \Delta_{V} /\epsilon_1) \text{ for } j \leq k, \text{ and } V_{kj} = V_{jk}, \label{Vterm} \\
w_j &\sim \Lap(0, \Delta_{w}/ \epsilon_2) \text{ for all $j$}, \label{Wterm} 
\end{align}
where $\Delta_{V}$ and $\Delta_{w}$ bound the $L_1$ sensitivity of $V$ and $w$, respectively. The result is $(\epsilon_1 + \epsilon_2)$-DP. Because $X^TX = \sum_{i=1}^n \x_i \x_i^T$ and $X^T \y = \sum_{i=1}^n \x_i y_i$ are additive, we can take $\Delta_V = \sum_{j \leq k} \width(x_j)\cdot\width(x_k)$ and $\Delta_w = \sum_j \width(x_j)\cdot\width(y)$, where $\width(x_j)$ and $\width(y)$ are widths of the $j$th feature and response variable, respectively, which are enforced by the modeler.

For confidence intervals, we will also need a private estimate of $\sigma^2$: let $\hat{\sigma}^2 = (n-p)^{-1}\sum_{i=1}^n (y_i - \hat{\beta}^T \x_i)^2 + \Lap(0, \Delta_z / \epsilon_3)$ where $\Delta_z = \width((y - \hat{\beta}^T \x)^2)$. \ds{TODO} The released values for SSP-OLS are then $(X^TX + V, X^T \y + w, \hat\beta, \hat\sigma^2)$, which satisfy $(\epsilon_1 + \epsilon_2 + \epsilon_3)$-DP.

\paragraph{Limitations of parametric bootstrap for private regression}
The parametric bootstrap is more difficult to apply to regression problems in a private setting due to the covariates. It is typical to bootstrap conditioned on $X$, which means simulating new response variables $\y$ from a parametric distribution $p(\y | X; \hat{\beta}, \hat{\sigma}^2)$, where $\hat\beta$ and $\hat\sigma^2$ are (privately) estimated parameters, and a fully parametric distribution $p(u; \sigma^2)$ is assumed for errors. A bootstrap replicate would look like $\hat\beta^* = (X^TX)^{-1}X^T\y^*$ with $\y^* = X\hat\beta + \u^*$ and $\u^*$ simulated form the error distribution. 
The challenge is that $X$ is accessed to generate each replicate, so to make it differentially private would require additional randomization and consume privacy budget. An alternative would be to posit a model $p(\x; \theta)$ and perform the parametric bootstrap with respect to a joint model $p(\x, y ; \theta, \beta, \sigma^2)$, but the additional demand to model covariates is unappealing in a regression context.

\paragraph{Hybrid parametric bootstrap for OLS}
We propose a novel hybrid approach that avoids the need to repeatedly access covariate data or to model the covariate or error distributions explicitly. Conceptually, we use the part of the standard asymptotic analysis that ``works well'' to approximate the relevant statistics of the covariate data, and use the parametric bootstrap to deal with the noise added for privacy.
Following standard analysis for OLS, we can substitute $\y = X \beta + \u$ in \eqref{eq:privnorm} and scale terms to get:
\begin{equation}
\begin{split}
\label{eq:beta_hat}
\sqrt{n}\hat{\beta}_n  &= 
\sqrt{n}\left(\check Q_n + \frac{1}{n}V \right)^{-1} \check Q_n  \beta  +
\\ &+ \left(\check Q_n + \frac{1}{n}V \right)^{-1} \left(\blue{\frac{1}{\sqrt{n}} X^T \mathbf{u}} + \frac{1}{\sqrt{n}} w \right), \\ \,\, \check Q_n &= \frac{1}{n} X^T X.
\end{split}
\end{equation}

This expression is instructive to see the different sources of randomness that contribute to the variability of $\hat\beta_n$:
the terms $\check Q_n = \frac{1}{n}X^TX$ and $\frac{1}{\sqrt{n}}X^T \u$ are due to data variability, and $\frac{1}{n}V$ and $\frac{1}{\sqrt{n}}w$ are due to privacy.
We form a bootstrap estimator $\hat\beta^*_n$ that treats $(\hat\beta_n, \beta)$ analogously to $(\hat\beta^*_n, \hat\beta_n)$ and simulates the different sources of variability using the best available information about their distributions:
\begin{equation}
\begin{split}
\label{eq:beta_hat_star}
  \sqrt{n}\hat{\beta}^*_n &= \sqrt{n}\left(\hat{Q}_n + \frac{1}{n}V^* \right)^{-1} \hat Q_n \hat{\beta}_n + 
  \\ &+ \left(\hat{Q}_n + \frac{1}{n}V^* \right)^{-1} \left(\blue{Z^*_n} + \frac{1}{\sqrt{n}} w^* \right), \\
  \notag \blue{Z_n^*} &\sim \N(0, \hat\sigma^2_n \hat Q_n), \quad V^* \sim P_V, \quad w^* \sim P_w, \\
  \notag \hat Q_n &= \frac{1}{n}X^TX + \frac{1}{n}V. 
 \end{split}
\end{equation}

All privacy terms in Eq.~\eqref{eq:beta_hat} are simulated from their exact distributions in Eq.~\eqref{eq:beta_hat_star}. The variables $\check Q_n$ and $\hat Q_n$ represent approximations of $Q = \E[\x\x^T]$ available to the corresponding estimator. Both quantities converge in probability to $Q$. Our choice not to simulate variability in these estimates due to the covariates is analogous to the ``fixed $X$'' bootstrap strategy for regression problems~\citep{fox2002r}; we \emph{do} simulate the variability due to privacy noise added to the estimates. The blue terms represent contributions to estimator variability due to interactions between covariates and unobserved noise variables. In a traditional bootstrap, we might simulate this term in Eq.~\eqref{eq:beta_hat_star} as $\frac{1}{\sqrt{n}}X^T \u^*$ where $\u^*$ are simulated errors, but, as described above, we do not wish to access $X$ within the bootstrap procedure. Instead, because we know $\frac{1}{\sqrt{n}}X^T \u \toD \N(0, \sigma^2 Q)$ by the central limit theorem,\footnote{This is a standard result of OLS asymptotics and is expected to be accurate for modest sample sizes.} and because $\sigma^2$ and $Q$ are estimable, we simulate this term directly from the normal distribution with estimated parameters. 

\begin{restatable}{theorem}{OLS}\label{OLS}
  \label{thm:OLS}
  The private estimator satisfies $\sqrt{n}(\hat\beta_n - \beta) \toD \N(0, \sigma^2 Q^{-1})$ and the bootstrap estimator $\hat\beta^*_n$ is consistent in the sense of Lemma~\ref{lem:consistency}.
\end{restatable}

\paragraph{OPS}
\label{sec:OPS}
\cite{Dimitrakakis:2014aa}, \cite{Wang:2015aa} and \cite{foulds2016theory} used the idea of sampling from a Bayesian posterior distribution to obtain a differentially private point estimate.
One Posterior Sampling (OPS), which releases one
sample from the posterior, is a special case of the exponential mechanism, and the corresponding estimator is near-optimal for parametric learning \citep{Wang:2015aa}.
The parametric bootstrap applies easily to OPS estimators and produces well calibrated intervals (Figure~\ref{fig:fig2}).
We expect the asymptotic analysis of \cite{Wang:2015aa} can be adapted to prove asymptotic equivariance, and hence parametric bootstrap consistency, for OPS, but do not give a formal proof.

\section{Related work}
\label{sec:related}
A number of prior works have studied private confidence intervals for different models~\citep{d2015differential,karwa2017finite,Sheffet:2017aa,Barri:2018,gaboardi2018locally,brawner2018bootstrap,du2020differentially}.
\citet{smith2011privacy} showed that a broad class private estimators based on subsample \& aggregate~\citep{nissim2007smooth} are asymptotically normal.
\cite{d2015differential} proposes an algorithm based on subsample \& aggregate to approximate the variance of a private estimator (see Section~\ref{significance}).
The topics of differentially private hypothesis testing \citep{vu2009differential,solea2014differentially,gaboardi2016differentially,couch2019differentially}
and Bayesian inference~\citep{Williams:2010aa,Dimitrakakis:2014aa,Wang:2015aa,foulds2016theory,ZZhang:2016aa,heikkila2017differentially,bernstein2018differentially,bernstein2019differentially} are also related, but the specific considerations differ somewhat from confidence interval construction. 
Finding practical and general-purpose algorithms for differentially private confidence intervals has been identified as an important open problem~\citep{opendp}.

The confidence interval approach of \citet{wang2019differentially} applies to any model fit by empirical risk minimization with objective or output perturbation and is similar to the asymptotic methods we compare to in Section~\ref{sec:experiments}.
\citet{evans2019statistically} also give a general-purpose procedure based on subsample \& aggregate (S\&A)~\citep{nissim2007smooth} with normal approximations.
This method also uses S\&A for the point estimates.
We compare to a similar variant of S\&A in Section~\ref{sec:experiments}.
\cite{wang2018statistical} study statistical approximating distributions for differentially private statistics in a general setting.


\citet{brawner2018bootstrap} use the \emph{non-parametric} bootstrap in a privacy context to estimate standard errors ``for free'' (at no additional cost beyond mean estimation) in some settings.
Other methods most similar to the our work on the parametric bootstrap were discussed in more detail in Sec.~\ref{significance}.

Prior methods to construct confidence intervals for private linear regression include \citep{Sheffet:2017aa,Barri:2018}.


\section{Experiments}\label{sec:experiments}
\begin{figure}[h!]
    \centering
    \includegraphics[width=1.02\linewidth]{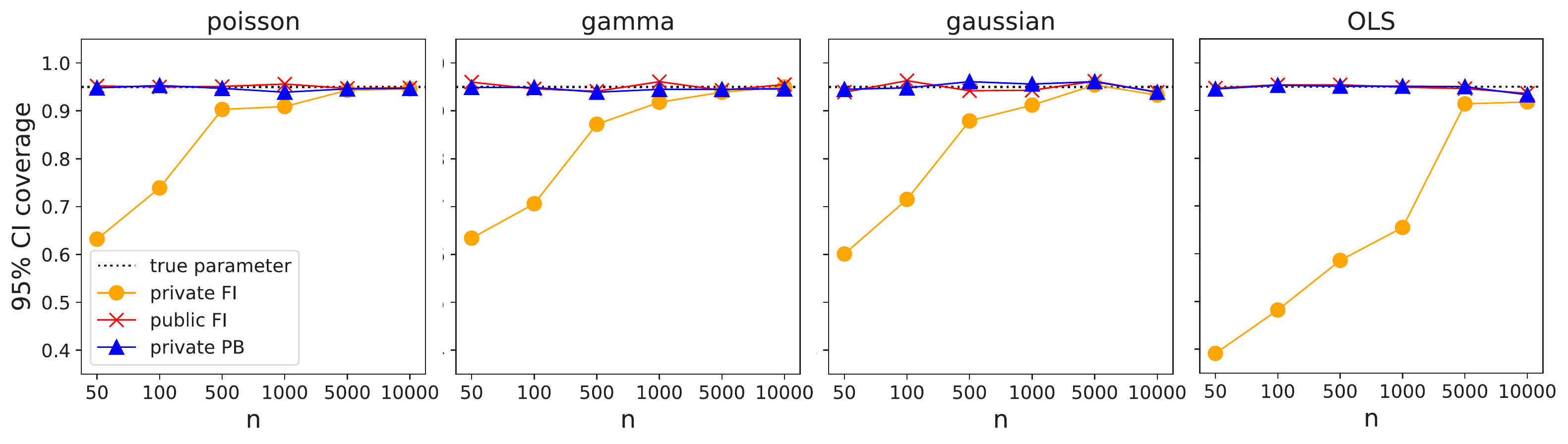}
    \caption{Observed vs nominal coverage of 95\% CIs for different distributions for different $n$.}
    \label{fig:1a}
\end{figure}

\begin{figure*}
    \centering
    \includegraphics[width=\linewidth]{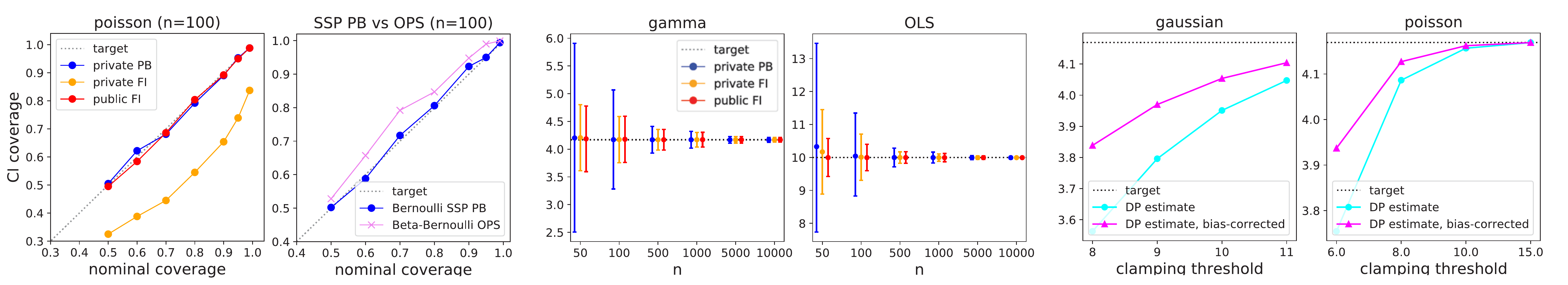}
    \caption{Left pair. (i) observed vs. nominal coverage for different coverage levels, for a Poisson, n=100, $\epsilon = 0.5$. (ii) same plot, comparing the OPS method \citep{foulds2016theory} and the parametric bootstrap for Bernoulli estimation. Center pair: average CI widths for different $n$ for (i) Gamma and (ii) OLS (other distributions give qualitatively similar results). Width of the private bootstrap CIs approaches that of the public CIs as $n \rightarrow \infty$. Right pair. (i) private and bias-corrected private estimates for a Poisson clamped at varying right-tail thresholds. (ii) same for a Gaussian clamped at $-10$ on the left tail and at varying thresholds on the right tail.}
    \label{fig:fig2}
\end{figure*}


We design synthetic experiments to demonstrate our proposed methods for differentially private confidence interval estimation.

First, we evaluate the performance of private parametric bootstrap CIs vs a baseline method (``Fisher CIs'') based on asymptotic normality of the private estimator and described in more detail below. Performance is measured by how well the coverage of the private CIs matches the nominal coverage. For all models, we also include Fisher CIs of non-private estimators for comparison. 

Second, we demonstrate the bias-correction procedure in Sec.~\ref{CIsBiasCorrection} in the case of Gaussian and Poisson distributions with data points clamped to different thresholds, which introduces estimation bias. These results show the effectiveness of the parametric bootstrap at approximating and mitigating the bias of differentially private estimates when sensitivity is bounded by forcing data to take values within given bounds.

Third, we compare parametric bootstrap CIs to another general purpose method to construct confidence intervals based on subsample \& aggregate~\citep{nissim2007smooth,smith2011privacy,d2015differential}.

Finally, the appendix includes additional experiments exploring a broader range of settings and performance metrics. These include: multivariate distributions, the effect of varying $\epsilon$, and measurements of the upper- and lower-tail CI failures. We aim for private CIs to be as tight as possible while providing the correct coverage: in the appendix, we also compare the width of our intervals with that of intervals from existing methods for the specific case of Gaussian mean estimation of known variance. 

\newcommand{\myparagraph}[1]{\smallskip\noindent \textbf{#1}\ }

\myparagraph{Baseline: ``Fisher CIs''}
As a byproduct of our consistency analysis we also derive asymptotic normal distributions of the private estimators for both exponential families (Theorem~\ref{MLEconsistency}) and OLS (Theorem~\ref{thm:OLS}).
In each case, we obtain a private, consistent estimate $\hat \sigma^2_j$ of the $j$th diagonal entry of the inverse Fisher information matrix of the private estimator $\hat\theta_n$, and then construct the confidence interval for $\theta_j$ as
\begin{equation}
C_n = \big[\hat{\theta}_{n,j} - z_{\alpha/2} \hat{\sigma}_j, \hat{\theta}_{n,j} + z_{\alpha/2} \hat{\sigma}_j\big],
\label{standardCI}
\end{equation}
where $z_{\gamma}$ is the $1-\gamma$-quantile of the standard normal distribution.
For exponential families, the Fisher information is estimated via plug-in estimation with the private estimator $\hat\theta_n$. For OLS, it is estimated via plugging in private estimates $\hat Q_n = \frac{1}{n} X^TX + \frac{1}{n}V$ and $\hat \sigma^2_n$, which are both released by the SSP mechanism.  
For non-private Fisher CIs, we follow similar (and very standard) procedures with non-private estimators. 

\myparagraph{Exponential families}
We use synthetic data sets drawn from different exponential family distributions.
Given a family, true parameter $\theta$, and data size $n$, a data set is drawn from $P_{\theta}$. We release private statistics via SSP with the Laplace mechanism. To simulate the modeler's domain knowledge about the data bounds, we draw a \emph{separate} surrogate data set of size 1000 drawn from the same distribution, compute the data range and use it to bound the width of each released statistic. For private estimation, sampled data is clamped to this range. Private $\hat{\theta}$ is computed from the privately released statistics using \ref{SSPMLE}. For the parametric bootstrap CIs, we implement Algorithm \ref{alg:bootstrap} and compute Efron's percentile intervals (see Table~\ref{table:CI-methods}). The output coverage is computed over $T=1000$ trials.

Results are shown in Figures~\ref{fig:1a} and \ref{fig:fig2}.
For the parametric bootstrap, actual coverage closely matches the nominal coverage, even for very small $n$. Coverage of private Fisher CIs is too low until $n$ becomes large, due to the fact that it ignores privacy noise. The bootstrap procedure correctly accounts for the increased uncertainty due to the privacy by enlarging the CIs. The width of the bootstrap intervals approaches the width of the baseline Fisher intervals as $n \rightarrow \infty$. In the appendix, we show that the coverage failures are balanced between left and right tails and examine the effect of increasing $\epsilon$ (which reduces privacy noise and has the same qualitative effect as increasing $n$).


\myparagraph{Linear regression}
We follow a very similar procedure for OLS. Data is generated with $x_j \sim \text{Unif}([-5, 5])$ for all $j$ and errors are $u_i \sim \text{Unif}[-10, 10]$; bounds on $y$ are passed as inputs ($[-150, 150]$) and assumed known. Observed values of $y$ exceeding the given bounds are dropped. These bounds are also used to compute widths for the sensitivity. Private coefficients are estimated with SSP-OLS and bootstrap CIs are constructed via Efron's percentile method. The results are shown in Fig.~\ref{fig:1a}.

\myparagraph{Bias correction}
In the case of distributions with infinite support, one option to bound the sensitivity is to clamp or truncate the data to given bounds. These procedures may induce estimation bias. As discussed in Sec~\ref{CIsBiasCorrection}, the parametric bootstrap can be used to approximate this bias and mitigate it. We demonstrate bias correction on the private estimates and CIs of a Poisson and Gaussian distribution where data is clamped on the right tail at different thresholds (Fig.~\ref{fig:fig2}).


\myparagraph{Comparison with subsample \& aggregate}
We compare the parametric bootstrap CIs with the intervals obtained via a subsample \& aggregate (S\&A) algorithm.
We adapted the S\&A procedure of \cite{d2015differential} for privately estimating standard errors to compute confidence intervals; see Algorithm~\ref{SACI} in the appendix. We compare the accuracy of point estimates and 95\% CIs for the mean of a Gaussian of known variance.  We found that the parametric bootstrap provides more accurate point estimates and better calibrated, tighter CIs than S\&A (Figure~\ref{fig:sa1}).

\begin{figure}[h!]
    \centering
    \includegraphics[scale=0.4]{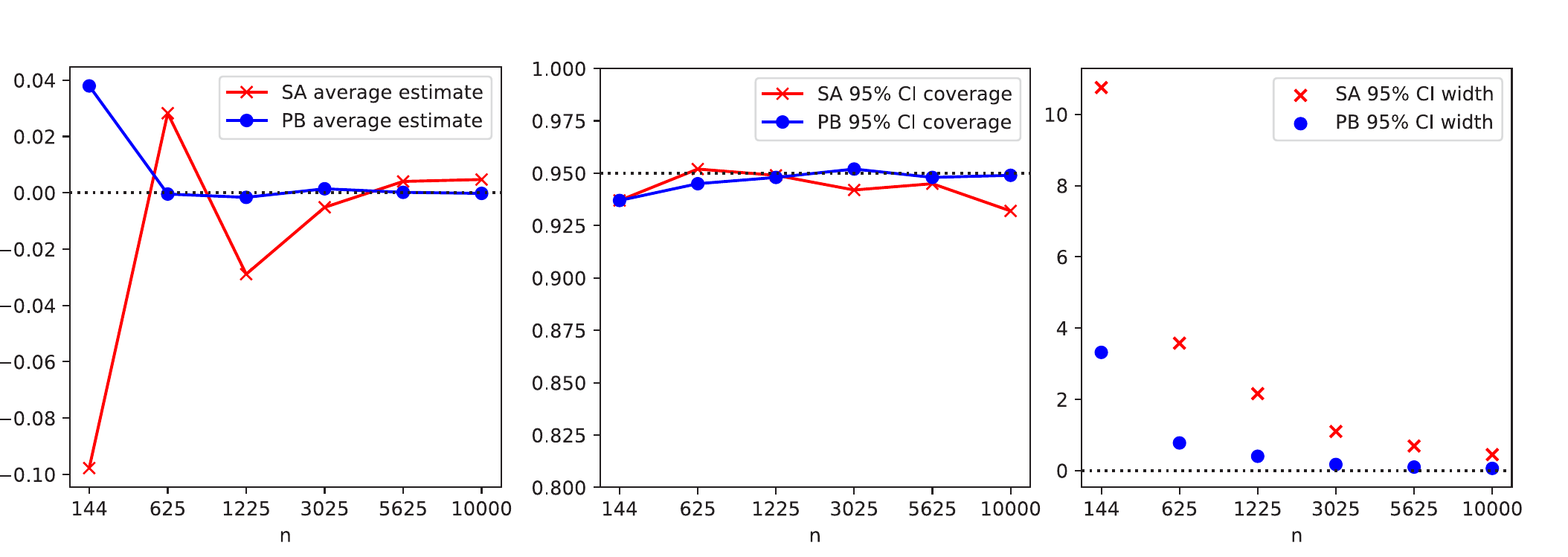}
    \caption{Point estimates (left), 95\% CI coverage (center) and average CI width (right) of the S\&A method  (see Algorithm~\ref{SACI} in appendix) vs  parametric bootstrap for the mean of a Gaussian of known variance. Settings: $\epsilon = 0.5$, $\theta=0$, $\sigma=1, (x_{min}, x_{max}) = (-20, +20), (L_{min}, L_{max}) = (-10, +10), var_{max} = 50$.}
    \label{fig:sa1}
\end{figure}

\section{Conclusion}
The parametric bootstrap is useful and effective to construct consistent differentially private confidence intervals for broad classes of private estimators, including private linear regression, for which we present a novel adaptation to avoid accessing the covariate data many times. The parametric bootstrap yields confidence intervals with good coverage even at modest sample sizes, and tighter than the ones based on subsample \& aggregate or other general methods. It can be used with any privacy mechanism, and can help mitigate differentially private estimation bias.

\newpage
\bibliography{bibliography}

\newpage
\onecolumn
\appendix

\section{Proof of Claim~\ref{claim:additive}}\label{sec:proofs} 

\additive*
\begin{proof}
Since $X$ and $X'$ differ in exactly one element and $f$ is additive, $f(X) - f(X') = g(x)-g(x')$ for some elements $x, x'$ in the data domain. The absolute value of the $j$th output $g_j(x)-g_j(x')$ is bounded by $\width(g_j) = \max_{x^*} g_j(x^*) - \min_{x^*} g_j(x^*)$. The $L_1$ sensitivity $\|f(X) - f(X')\|_1 = \|g(x)-g(x)'\|_1$ is therefore at most the sum of the widths.
\end{proof}

\section{Proofs for Bootstrap Theory}

\PBconsistency*

This theorem is a simplified version of the result of~\cite{beran1997diagnosing}. We give a self-contained proof. See also~\cite[][Problem 23.5]{van2000asymptotic}.

\begin{proof}
The distribution of $\sqrt{n}(\hat \tau_n - \tau)$ under $\P_\theta$ is $H_n(\theta)$, which, by asymptotic equivariance, converges to $H(\theta)$. In the parametric bootstrap, the distribution of $\sqrt{n}(\hat\tau^*_n - \hat\tau_n)$ conditional on $\hat\theta_n = \theta + h_n/\sqrt{n}$ is $H_n(\theta + h_n/\sqrt{n}))$, and, by asymptotic equivariance, $H_n(\theta + h_n/\sqrt{n})) \toD H(\theta)$ if $h_n$ is convergent. Since $H(\theta)$ is continuous this is equivalent to saying that, for all convergent sequences $h_n$ and all $t$
\begin{equation}
\label{eq:converge-deterministic}
\P^*_n\left( \sqrt{n}(\hat\tau^*_n - \hat\tau_n) \leq t \mid \hat\theta_n = \theta + h_n/\sqrt{n}\right) \rightarrow F_\theta(t).
\end{equation}
where $F_\theta$ is the CDF of $H(\theta)$.
Now, let $\hat h_n = \sqrt{n}(\hat\theta_n - \theta)$ so that $\hat\theta_n = \theta + \hat h_n / \sqrt{n}$. By assumption, $\hat h_n \toD J(\theta)$ and is therefore $O_P(1)$. Therefore, by Lemma~\ref{lem:seq-conv}, Eq.~\eqref{eq:converge-deterministic} implies 
$$
\P_n^*\left( \sqrt{n}(\hat\tau^*_n - \hat\tau_n) \leq t \mid \hat\theta_n \right) \rightarrow F_\theta(t) \text{ in $\P_\theta$-probability}
$$
and the result is proved.
\end{proof}

\begin{lemma}\label{lem:seq-conv}
Suppose $g_n$ is a sequence of functions such that $g_n(h_n) \to 0$ for any fixed sequence $h_n = O(1)$. Then $g_n(\hat{h}_n) \toP 0$ for every random sequence $\hat{h}_n = O_P(1)$.
\end{lemma}

\begin{proof}
Fix $\epsilon, \delta > 0$. We wish to show, for large enough $n$, that
$$\Pr\big[|g_n(\hat{h}_n)| > \epsilon \big] < \delta.$$
Since $\hat{h}_n$ is $O_P(1)$, there is some $M$ such that, for all $n$,
$$
\Pr\big[\|h_n\| > M\big] < \delta.
$$
By our assumption on $g_n$, there is some $N$ such that $|g_n(h)| < \epsilon$ for all $\|h\| \leq M, n > N$ (take the sequence $h_n \equiv h$ for each such $h$). Then, for $n > N$,
$$
\Pr\big[|g_n(\hat{h}_n)| > \epsilon\big] \leq \Pr\big[ \|h_n\| > M \big] < \delta.
$$
\end{proof}

\section{Proofs for Exponential Families}
\label{sec:proofs-expfam}
\begin{lemma}\label{lemma:noisevar}
Let $w$ be any random variable with mean zero and finite variance. For any $r > 0$, $\frac{1}{N^r}w \toP 0$.
\label{varw}
\end{lemma}
\ifthenelse{\equal{\showproofs}{1}}{
\begin{proof}
The variance of $N^{-r} w$ is equal to $N^{-2r} \Var(w)$, which goes to zero as $N \to \infty$. By Chebyshev's inequality, this implies that $N^{-r} w \overset{P}{\to} 0$.
\end{proof}}

\MLEconsistency*

Following standard practice, we will prove this for the case when $\theta$ is scalar; the generalization to vector $\theta$ is straightforward but cumbersome.
We first state the required (standard) regularity conditions. Let
\begin{equation*}
\begin{split}
\ell_n(\theta) &= \sum_{i=1}^n \ell(\theta; x_i) = \sum_{i=1}^n \big(\log p(x_i; \theta) - \log h(x)\big) \\&= \theta \sum_{i=1}^n T(x_i) - n A(\theta)
\end{split}
\end{equation*}
be the log-likelihood of a sample $x_{1:n}$ from the exponential family model using the definition of log-likelihood from Sec.~\ref{introduction}. Let $\ell(\theta) = \ell_1(\theta)$ be the log-likelihood of a single $x \sim p(x; \theta)$.

We assume the log-likelihood satisfies the conditions given in the book of \citet[Section 4.4.2]{davison2003statistical}. If it does, then we have the following
\begin{enumerate}[label=(F\arabic*)]
\item $\E_{\theta}[\ell'(\theta)] = 0$.
\item $\Var_{\theta}[\ell'(\theta)] = -\E_{\theta}[\ell''(\theta)] = I(\theta)$.
\item Given a sequence of estimators $\hat{\theta}_n \toP \theta$,  for all $\tilde{\theta}_n \in [\theta, \hat{\theta}_n]$, $\;\frac{1}{2\sqrt{n}} \ell'''_n(\Tilde{\theta}_n)(\hat{\theta}_n - \theta)^2 \toP 0$ 
\end{enumerate}
Facts (F1) and (F2) are well known exponential family properties.
Also recall that, for an exponential family,
\begin{enumerate}[label=(F\arabic*),resume]
  \item $I(\theta) = A''(\theta)$.
  \item $-\ell''(\theta)$ is \emph{deterministic} and equal to $I(\theta)$.
\end{enumerate}

\begin{proof}
Let $\lambda_n(\theta) = f\big(\theta, w + \sum_{i=1}^n T(x_i)\big)$ be the objective of the SSP-MLE optimization problem. We have
  \begin{equation*}
  \begin{aligned}
    \lambda_n(\theta)
    &= \theta \Big(w + \sum_{i=1}^n T(x_i)\Big) - n A(\theta) \\
    &= \theta w + \theta \sum_{i=1}^n T(x_i) - n A(\theta) \\
    &= \theta w + \ell_n(\theta)
  \end{aligned}
  \end{equation*}
  where $\ell_n(\theta)$ is the log-likelihood of the true sample. That is, the original objective $\ell_n(\theta)$ is perturbed by the linear function $\theta w$ to obtain $\lambda_n(\theta)$. The derivatives are therefore related as:
  \begin{align}
    \lambda_n'(\theta) &= w + \ell'_n(\theta), \label{eq:deriv1} \\
    \lambda_n^{(k)}(\theta) &= \ell^{(k)}(\theta), \quad k > 1. \label{eq:deriv2}
  \end{align}
  
  At the optimum $\hat{\theta}_n$, the first derivative of $\lambda_n$ is equal to zero. $\ell'(\theta)$ is a sum of i.i.d. terms with mean $0$ and variance $I(\theta)$, more specifically:
    $$
    \ell'(\theta) = \sum_{i=1}^n \left( T(x_i) - A'(\theta) \right)
    $$
    
  For asymptotic equivariance, we are interested in the sequence of estimators $\hat\theta_n$ when the ``true parameter'' follows the sequence $\theta + h_n/\sqrt{n}$.
  We follow the standard approach of writing the Taylor expansion of the first derivative about the true parameter $\theta + h_n/\sqrt{n}$:
  \begin{equation}
    0 = w + \ell'_n(\theta + h_n/\sqrt{n})  + \ell''_{n}(\theta + h_n/\sqrt{n})(\hat{\theta}_n - \theta - h_n/\sqrt{n}) + Z_n
\label{taylor}
  \end{equation}
  where we have used Eqs.~\eqref{eq:deriv1} and \eqref{eq:deriv2} to replace the derivatives of $\lambda$ on the right-hand side, and $Z_n = \tfrac{1}{2} \ell'''_n(\Tilde{\theta}_n)(\hat{\theta}_n - \theta - h_n/\sqrt{n})^2$ is the second-order Taylor term, with $\tilde{\theta}_n$ some point in the interval $[\theta + h_n/\sqrt{n}, \hat{\theta}_n]$.
\\

Multiply both sides of the equation by $\frac{1}{\sqrt{n}}$ and rearrange to get
\begin{equation*}
\begin{split}
  \sqrt{n}(\hat{\theta}_n - \theta - h_n/\sqrt{n})
  &= \frac{\frac{1}{\sqrt{n}} w + \frac{1}{\sqrt{n}} \ell'_n(\theta + h_n/\sqrt{n}) + \frac{1}{\sqrt{n}}Z_n}
  {-\frac{1}{n}\ell''_{n}(\theta + h_n/\sqrt{n})}
  \\&= \frac{\frac{1}{\sqrt{n}} w + \frac{1}{\sqrt{n}} \ell'_n(\theta + h_n/\sqrt{n}) + \frac{1}{\sqrt{n}}Z_n}
         {I(\theta + h_n/\sqrt{n})}
\end{split}
\end{equation*}
where in the second equality we used (F5). By Lemma~\ref{varw}, $\frac{1}{\sqrt{n}}w \toP 0$ and by (F3) $\frac{1}{\sqrt{n}}Z_n \toP 0$, so, by Slutsky,
\begin{equation}\label{eq:middle_term}
  \sqrt{n}(\hat{\theta}_n - \theta - h_n/\sqrt{n}) \toP
  \frac{\frac{1}{\sqrt{n}} \ell'(\theta + \frac{h_n}{\sqrt{n}})}{I(\theta + \frac{h_n}{\sqrt{n}})} = \underbrace{\frac{\ell'(\theta + \frac{h_n}{\sqrt{n}})}{\sqrt{n I(\theta + \frac{h_n}{\sqrt{n}})}}}_{(B)} \frac{1}{\sqrt{I(\theta + \frac{h_n}{\sqrt{n}})}}
\end{equation}

We know that under the regularity assumptions the Fisher information $I(\cdot)$ is a continuous function and so, since $\theta + \frac{h_n}{\sqrt{n}} \to \theta$, then by continuity ($I(\cdot)$ is deterministic):

\begin{align*}
    \left(I \left(\theta + \frac{h_n}{\sqrt{n}}\right) \right)^{-1/2} \to \left(I(\theta)\right)^{-1/2} 
\end{align*}

We now focus on the asymptotic behavior of $(B)$. We will use the fact that in exponential families, $\E_{\theta} T(x) = A'(\theta)$ and $\Var_\theta T(x) = A''(\theta) = I(\theta)$. For simplicity of notation, define:
$$
\mu_n = A'\left(\theta + \frac{h_n}{\sqrt{n}}\right).
$$

Define now the triangular array written in the following notation:

\begin{align*}
&T(x_1) - \mu_1 \hspace{5cm} x_1 \sim \mathbb{P}_{\theta + \frac{h_1}{1}}\\
&T(x_1) - \mu_2, T(x_2) - \mu_2 \hspace{3cm} x_{1:2} \overset{i.i.d.}{\sim} \mathbb{P}_{\theta + \frac{h_2}{\sqrt{2}}}\\
&T(x_1) - \mu_3, T(x_2) - \mu_3, T(x_3) - \mu_3 \hspace{1.2cm} x_{1:3} \overset{i.i.d.}{\sim} \mathbb{P}_{\theta + \frac{h_3}{\sqrt{3}}}\\
&... \hspace{7cm} ... \\
&T(x_1) - \mu_n, T(x_2) - \mu_n, ..., T(x_n) - \mu_n \hspace{0.6cm} x_{1:n} \overset{i.i.d.}{\sim} \mathbb{P}_{\theta + \frac{h_n}{\sqrt{n}}}
\end{align*}

Let's focus on the $n$-th row. By construction the sum over the $n$-th row is $S_n=\sum_{i=1}^n (T(x_i) - \mu_n) = \ell'(\theta + \frac{h_n}{\sqrt{n}})$, so the numerator of $(A)$. Each term in the $n$-th row has mean zero and:

$$
\sigma^2_n = \sum_{i=1}^n \Var[T(x_i) - \mu_n] = n I\left(\theta + \frac{h_n}{\sqrt{n}} \right).
$$

If for every $\epsilon > 0$ the following condition holds:

$$
\lim_{n \to \infty}  \frac{1}{\sigma^2_n} \sum_{i=1}^n \mathbb{E}\left[(T(x_i) - \mu_n )^2 \mathbf{1}\left( \abs{T(x_i) - \mu_n} \geq \epsilon \sigma_n \right)\right] = 0,
$$

then $S_n/\sigma_n \to \mathcal{N}(0,1)$ by the Lindeberg-Feller Central Limit Theorem. By plugging in the terms in the condition above we have that:

\begin{align*}
    \lim_{n \to \infty}  &\frac{1}{n I\left(\theta + \frac{h_n}{\sqrt{n}}\right)} \sum_{i=1}^n \mathbb{E}\left[(T(x_i) - \mu_n )^2 \mathbf{1}\left( \abs{T(x_i) - \mu_n} \geq \epsilon \sqrt{n}\sqrt{I\left(\theta + \frac{h_n}{\sqrt{n}} \right)} \right)\right] \\
    &= \lim_{n \to \infty} I\left(\theta + \frac{h_n}{\sqrt{n}}\right)^{-1} \mathbb{E}\left[(T(x_1) - \mu_n )^2 \mathbf{1}\left( \abs{T(x_1) - \mu_n} \geq \epsilon \sqrt{n}\sqrt{I\left(\theta + \frac{h_n}{\sqrt{n}}\right)} \right)\right]
\end{align*}

with the equality due to i.i.d. sampling within the row of the triangular array. Note that for any $x_1$,

$$
\lim_{n \to \infty} (T(x_1) - \mu_n )^2 \mathbf{1}\left( \abs{T(x_1) - \mu_n} \geq \epsilon \sqrt{n}\sqrt{I\left(\theta + \frac{h_n}{\sqrt{n}}\right)} \right) = 0,
$$

and that the integrand above is dominated by $(T(x_1) - \mu_n )^2$, which is integrable and finite, since $\mathbb{E}[(T(x_1) - \mu_n )^2]$ is finite. Hence by the dominated convergence theorem, the limit is zero and the condition is satisfied.\\

Going back to equation~\eqref{eq:middle_term}, we then have that

$$
\frac{\frac{1}{\sqrt{n}} \ell'(\theta + \frac{h_n}{\sqrt{n}})}{I(\theta + \frac{h_n}{\sqrt{n}})} \toD \mathcal{N}(0,1) \cdot \left(I(\theta)\right)^{-1/2} = \mathcal{N}(0,I(\theta)^{-1}),
$$
which proves that $\sqrt{n}(\hat{\theta}_n - \theta - h_n/\sqrt{n}) \rightsquigarrow \N(0, I(\theta)^{-1})$. Setting $h_n = 0$, it is straightforward to find that $\sqrt n (\hat\theta_n - \theta) \rightsquigarrow \mathcal{N}(0, I(\theta)^{-1})$. This proves that SSP-MLE is asymptotically equivariant.
\end{proof}

\section{Proofs for OLS}
\label{sec:proofs-ols}

\OLS*

Before proving the theorem, we give two lemmas. The first is standard and describes the asymptotics of the dominant term.
\begin{restatable}{lemma}{OLSXTu}
    \label{lem:XTu}
    Under the assumptions of the OLS model in Section~\ref{sec:OLS}, $\frac{1}{\sqrt{n}}X^T \u \toD \N(0, \sigma^2 Q)$.
\end{restatable}
\begin{proof}
Observe that $X^T \u = \sum_{i=1}^n \x_i u_i$  is a sum of iid terms, and, using the assumptions of the model in Section~\ref{introduction}, the mean and variance of the terms are $\E[\x_i u_i] = \E[\x_i]\E[u_i] = 0$ and $\Var(\x_i u_i) = \Var(\x u) = \E[\x u u \x^T ] = \E[u^2 \x \x^T] = \E[u^2] \E[\x \x^T] = \sigma^2 Q$.
The result follows from the central limit theorem.
\end{proof}

The theorem involves asymptotic statements about $\hat\beta_n$ and $\hat\beta^*_n$. The following lemma is a general asymptotic result that will apply to both estimators using Eqs.~\eqref{eq:beta_hat} and \eqref{eq:beta_hat_star}. 

\begin{lemma}
  \label{lem:ols-general}
  Define the function
  $$
  \B_n\{\check{Q}, \check{\beta}, \check{Z}, \check{V}, \check{w}\} = 
  \left(\check{Q}+ \frac{1}{n}\check{V} \right)^{-1} \check{Q} \check{\beta}
  + \left(\check{Q} + \frac{1}{n}\check{V} \right)^{-1} \left(\frac{1}{\sqrt{n}} \check{Z} + \frac{1}{n} \check{w} \right)
  $$
  and suppose the sequences $Q_n, \beta_n, Z_n, V_n, w_n$ are defined on a common probability space and satisfy
  \begin{enumerate}[label=(\roman*)]
    \item $Z_n \toD \N(0, \sigma^2 Q)$,
    \item  $Q_n \toP Q$,
    \item $\beta_n, V_n, w_n$ are all $O_P(1)$.
  \end{enumerate}
  
  Then
  $$
  \sqrt{n}\Big(\B_n\{Q_n, \beta_n, Z_n, V_n, w_n\} - \beta_n\Big) \toD \N(0, \sigma^2 Q^{-1}).
  $$
\end{lemma}
\begin{proof}
Substitute the sequences into $\B_n$ and rearrange to get
\begin{align}
 \sqrt{n}
 \left( \B_n - \beta_n\right) = \sqrt{n}\left( \left(Q_n +  \frac{1}{n}V_n \right)^{-1} Q_n - I \right)
 \beta_n 
+ 
 \sqrt{n}\left(Q_n + \frac{1}{n}V_n \right)^{-1} 
 \left(\frac{1} {\sqrt{n}}Z_n + \frac{1}{n}w_n \right)
 \label{eq:ols-pivot-general}
\end{align}
First, note that the sequences $\frac{1}{n}V_n$, $\frac{1}{\sqrt{n}} V_n$ and $\frac{1}{\sqrt{n}}w_n$, which will appear below, are all $o_P(1)$,  since $V_n$ and $w_n$ are $O_P(1)$.

The first term in Eq.~\eqref{eq:ols-pivot-general} converges to zero in probability. Specifically, a manipulation shows:

\begin{align*}
\sqrt{n}
\left( \left(Q_n + \frac{1}{n}V_n \right)^{-1} Q_n - I \right) \beta_n
&=  \left(Q_n + \frac{1}{n}V_n \right)^{-1} \left( -\frac{1}{\sqrt{n}}V_n\right) \beta_n \\
&= O_P(1)o_P(1)O_P(1) \\
& = o_P(1)
\end{align*}
For the first factor on the right side,  $(Q_n + \frac{1}{n}V_n)^{-1} \toP Q^{-1}$ (by Slutsky's theorem, since $Q_n \toP Q$ and $\frac{1}{n}V_n \toP 0$), and is therefore $O_P(1)$. For the second factor, we argued $-\frac{1}{\sqrt{n}} V_n= o_P(1)$. For the third factor, $\beta_n = O_P(1)$ by assumption. 

The second term in Eq.~\eqref{eq:ols-pivot-general} converges in distribution to $\N(0, \sigma^2 Q)$. Rewrite it as
$$
\left(Q_n + \frac{1}{n}V_n \right)^{-1}
\left(Z_n + \frac{1}{\sqrt{n}}w_n \right).
$$
We already argued that  $(Q_n + \frac{1}{n}V_n)^{-1} \toP  Q^{-1}$ and $\frac{1}{\sqrt{n}} w_n \toP 0$. By assumption, $Z_n \toD \N(0, \sigma^2 Q)$. Therefore, by Slutsky's theorem, the entire term converges in distribution to $Q^{-1}\N(0, \sigma^2 Q) = \N(0, \sigma^2 Q^{-1})$. 
\end{proof}

We are ready to prove the Theorem~\ref{thm:OLS}.

\begin{proof}[Proof of Theorem~\ref{thm:OLS}]

We first wish to show that $\sqrt{n}\left(\hat{\beta}_n - \beta \right) \toD \N(0, \sigma^2Q^{-1})$. To see this, write $\hat{\beta}_n  = \B_n\left\{ \frac{1}{n}X^T X, \beta, \frac{1}{\sqrt{n}}X^T \u, V, w\right\}$ and apply Lemma~\ref{lem:ols-general}. It is easy to verify that the sequences satisfy the conditions of the lemma.

Next, we wish to show that the bootstrap estimator is consistent. By Lemma~\ref{lem:consistency}, it is enough to show that $\sqrt{n}(\hat\beta^*_n - \hat{\beta}_n) \toD \N(0, \sigma^2 Q^{-1})$ conditional on $\omega$ in $\P_\theta$-probability, where $\theta = (\beta, \sigma^2, Q)$ and $\P_\theta$ is the common probability space of the data and privacy random variables, represented by $\omega$. The bootstrap variables $Z^*_n, V^*, w^*$ correspond to the inner measure $\P^*_n$. 
Define $\hat{Q}_n = \frac{1}{n}X^T X + \frac{1}{n}V$. Observe that Eq.~\eqref{eq:beta_hat_star} is equivalent to
\begin{equation}
\hat\beta^*_n = \B_n\{\hat{Q}_n, \hat{\beta}_n, Z^*_n, V^*, w^* \} \ \ 
\text{under } Z^*_n \sim \N(0, \hat{\sigma}^2_n \hat{Q}_n), V^* \sim F_V, w^* \sim F_w,
\label{eq:beta_hat_star_claim1}
\end{equation}
and $(\hat{Q}_n, \hat{\beta}_n, \hat{\sigma}^2)$ are consistent estimators and hence converge in $\P_\theta$-probability to $(Q, \beta, \sigma^2)$.

We can't apply Lemma~\ref{lem:ols-general} directly to Eq.\eqref{eq:beta_hat_star_claim1} because this expression mixes random variables from the outer space ($\hat{Q}_n, \hat{\beta}_n, \hat{\sigma}^2_n$) and inner space $(Z^*_n, V^*, w^*)$. 
Instead, we temporarily reason about a \emph{deterministic} sequence $(Q_n, \beta_n, \sigma^2_n) \to (Q, \beta, \sigma^2)$. Then, by Lemma~\ref{lem:ols-general} applied to the inner probability space,
\begin{equation}
\label{eq:deterministic}
\begin{split}
\sqrt{n}\Big(\B_n\{Q_n, \beta_n, Z^*_n, V^*, w^*\} - \beta_n\Big) \toD \N(0, \sigma^2Q^{-1}) \\
\text{under } Z^*_n \sim \N(0, \sigma^2_n Q_n), V^* \sim P_V, w^* \sim P_w
\end{split}
\end{equation}
The conditions of Lemma~\ref{lem:ols-general} can easily be checked. In particular, we have $Z^*_n \toD \N(0, \sigma^2 Q)$.

We can restate the result Eq.~\eqref{eq:deterministic} as follows: for any fixed sequence $(Q_n, \beta_n, \sigma^2_n) \to (Q, \beta, \sigma^2)$ and all $t$,
$$
\P^*_n\left( \sqrt{n} (\hat\beta^*_n-\hat\beta_n) \leq t \mid \hat Q_n = Q_n, \hat\beta_n = \beta_n, \hat\sigma^2_n = \sigma_n\right) \to F(t)
$$
where $F$ is the CDF of $\N(0, \sigma^2 Q^{-1})$. Lemma~\ref{lem:key_lemma_2} below now implies that
$$
\P^*_n\left( \sqrt{n} (\hat\beta^*_n-\hat\beta_n) \leq t \mid \hat Q_n, \hat\beta_n, \hat\sigma^2_n\right) \to F(t) \text{ in $\P_\theta$-probability},
$$
and the theorem is proved.
\end{proof}

\begin{lemma}\label{lem:key_lemma_2}
Let $g_n: \R^k \to \R^\ell$ be a sequence of functions such that $g_n(h_n) \to c$ for any deterministic sequence $h_n \to h$. Then $g_n(\hat{h}_n) \toP c$ for any random sequence  $\hat{h}_n \toP h$.
\end{lemma}

\begin{proof}
Take $c = 0$ without loss of generality, let $\|\cdot \|$ be any norm and $d(x, y) = \|x - y\|$.
Fix $\epsilon > 0$. It must be the case that
\begin{equation}
\label{eq:converge}
\exists\, \delta > 0, n_0 \in \mathbb{N} \text { such that:  }\quad d(h', h) < \delta \implies  \|g_n(h')\| < \epsilon, \forall n \geq n_0.
\end{equation}

Otherwise, we can construct a convergent sequence $h_n \to h$ with $\limsup_{n \to \infty} \|g_n(h_n)\| \geq \epsilon$, which violates the conditions of the Lemma.\footnote{If Eq.~\eqref{eq:converge} is not true, then for all $\delta > 0$ and $n_0 \in \mathbb{N}$, there exists $h'$ such that $d(h', h) < \delta$ and  $\|g_n(h')\| \geq \epsilon$ for some $n \geq n_0$. Then we can construct a sequence $h_n \to h$ as follows. Let $\delta_k$ be any sequence such that $\delta_k \to 0$. Set $n_0 = 0$, and, for $k \geq 1$, select $h'$ such that $d(h', h) < \delta_k$ and $\|g_{n'}(h')\| \geq \epsilon$ for some $n' \geq n_{k-1} + 1$. Set $h_n = h'$ for all $n \in \{n_{k-1}+1, \ldots, n'\}$ and let $n_k = n'$. This sequence satisfies $h_n \to h$ but $g_{n_k}(h_{n_k}) \geq \epsilon$ for all $k$, so it is not true that $g_n(h_n) \to 0$. This contradicts the assumptions of the lemma, so Eq.~\eqref{eq:converge} must be true.}

Now, suppose $\hat{h}_n \toP h$. Then, for $n \geq n_0$, by Eq.~\eqref{eq:converge},
$$
\Pr\left[\|g_n(\hat h_n)\| > \epsilon\right] \leq \Pr\left[d(\hat h_n, h) > \delta\right].
$$
Therefore
$$
\lim_{n \to \infty} \Pr\left[\|g_n(\hat{h}_n)\| > \epsilon\right] 
\leq \lim_{n \to \infty} \Pr\left[  d(\hat h_n, h) > \delta \right] = 0,
$$
which proves the result.
\end{proof}

\newpage
\section{Subsample \& Aggregate}\label{sec:additionalexperiments} 
\begin{algorithm}
\caption{Subsample\&Aggregate}\label{SACI}
\hspace*{\algorithmicindent} \textbf{Input} $X, M, x_{min}, x_{max}, L_{min}, L_{max}, var_{max}, \epsilon, \alpha$
\begin{algorithmic}[1]
\Procedure{SubsampleAndAggregate}{}
\State $X_1, ... , X_M \gets \textbf{subsample}(X, M)$
\State $L^*_{min}, L^*_{max} \gets \frac{L_{min}}{\sqrt{N/M}}, \frac{L_{max}}{\sqrt{N/M}}$
\State $var^*_{max} \gets \frac{var_{max}}{N/M}$
\For {$i = 1, ..., M$} 
\State $\hat c_i \gets \textbf{clamp}(\mathcal{A}(X_i), L^*_{min}, L^*_{max})$
\EndFor
\State $\Delta_1 \gets \frac{|L^*_{max} - L^*_{min}|}{M}$
\State $\hat\theta_{DP} \gets \frac{1}{M} \sum_{i=1}^{M} \hat c_i + 
\text{Lap}(0, \frac{\Delta_1}{ \epsilon / 2})$
\For {$i = 1, ... , M$}
\For {$b = 1, ... , B$}
\State $X_{i,b} \gets \textbf{resample}(X_i, \lfloor \frac{N}{M} \rfloor, \text{replace=True})$
\State $\hat c_{i,b} \gets \textbf{clamp}(\mathcal{A}(X_{i,b}), L^*_{min}, L^*_{max})$
\EndFor
\State $\hat{var}_{\hat{c}_i} \gets \textbf{clamp}(\Var(\hat c_{i, 1:B}), 10^{-6}, var^{*}_{max})$
\EndFor
\State $\Delta_2 \gets var^{*}_{max}/M$ 
\State $\hat{var}_{\hat c} \gets \frac{1}{M} \sum_{i=1}^{M} \hat{var}_{\hat{c}_i} + \text{Lap}(0, \frac{\Delta_2}{\epsilon / 2}) $
\State $\hat{var}_{DP} \gets  \frac{1}{M}\hat{var}_{\hat c} + \Var(\text{Lap}(0, \frac{\Delta_1}{ \epsilon / 2})) $
\State $\text{CI}_{DP} \gets [\hat\theta_{DP} - z_{\frac{\alpha}{2}} \sqrt{\hat{var}_{DP}}, \hat\theta_{DP} + z_{\frac{\alpha}{2}}\sqrt{\hat{var}_{DP}}]$
\EndProcedure
\Return $\hat\theta_{DP}, \text{CI}_{DP}$
\end{algorithmic}
\end{algorithm}

\newpage
\section{Additional Experiments}\label{sec:additionalexperiments} 
\begin{figure*}[h!]
    \centering
 \hspace*{-0.5cm}
\begin{subfigure}{.71\textwidth}
  \centering
 \includegraphics[width=1\linewidth]{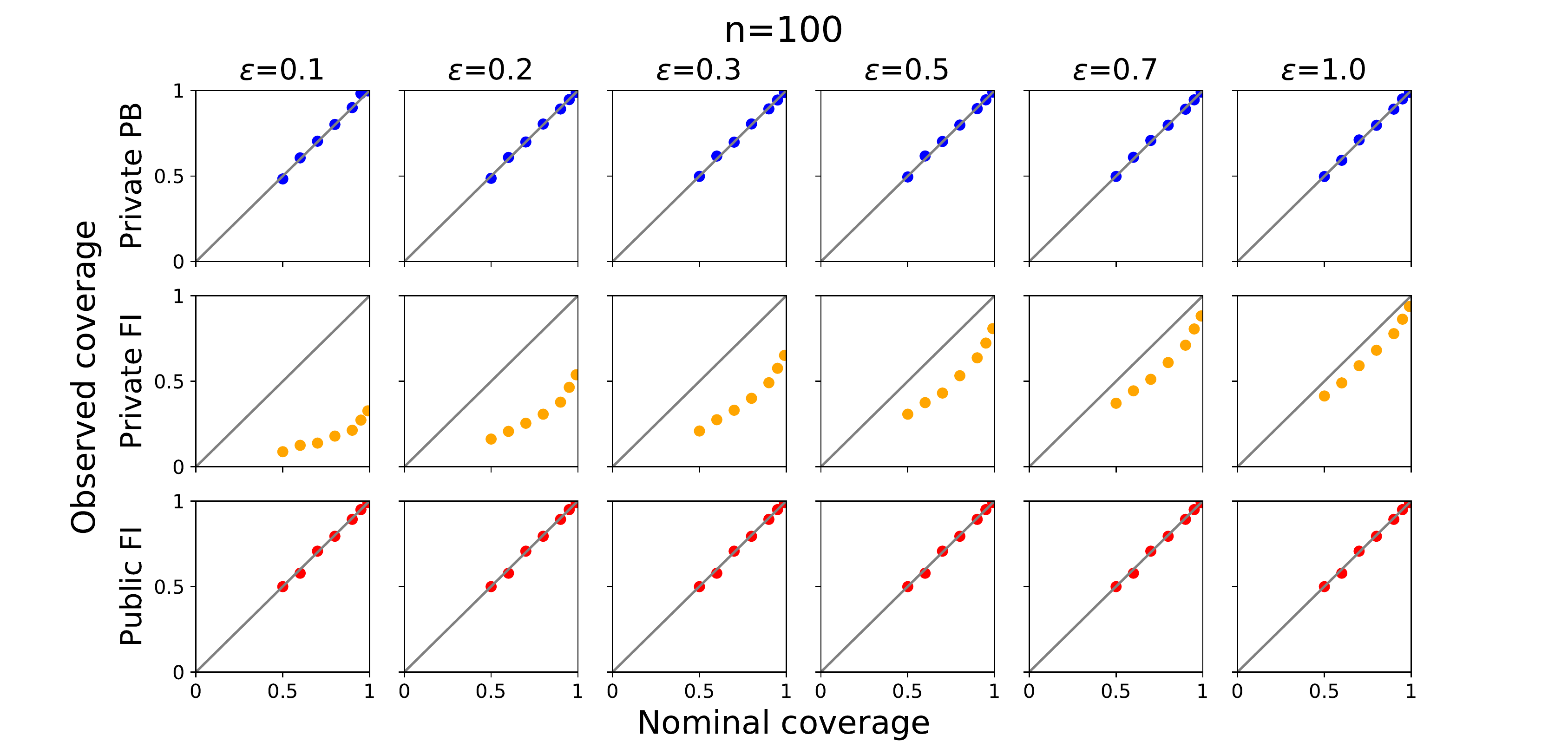}
  \caption{}
\end{subfigure}%
 \hspace*{-0.8cm}
\begin{subfigure}{.36\textwidth}
  \centering
  \includegraphics[width=1\linewidth]{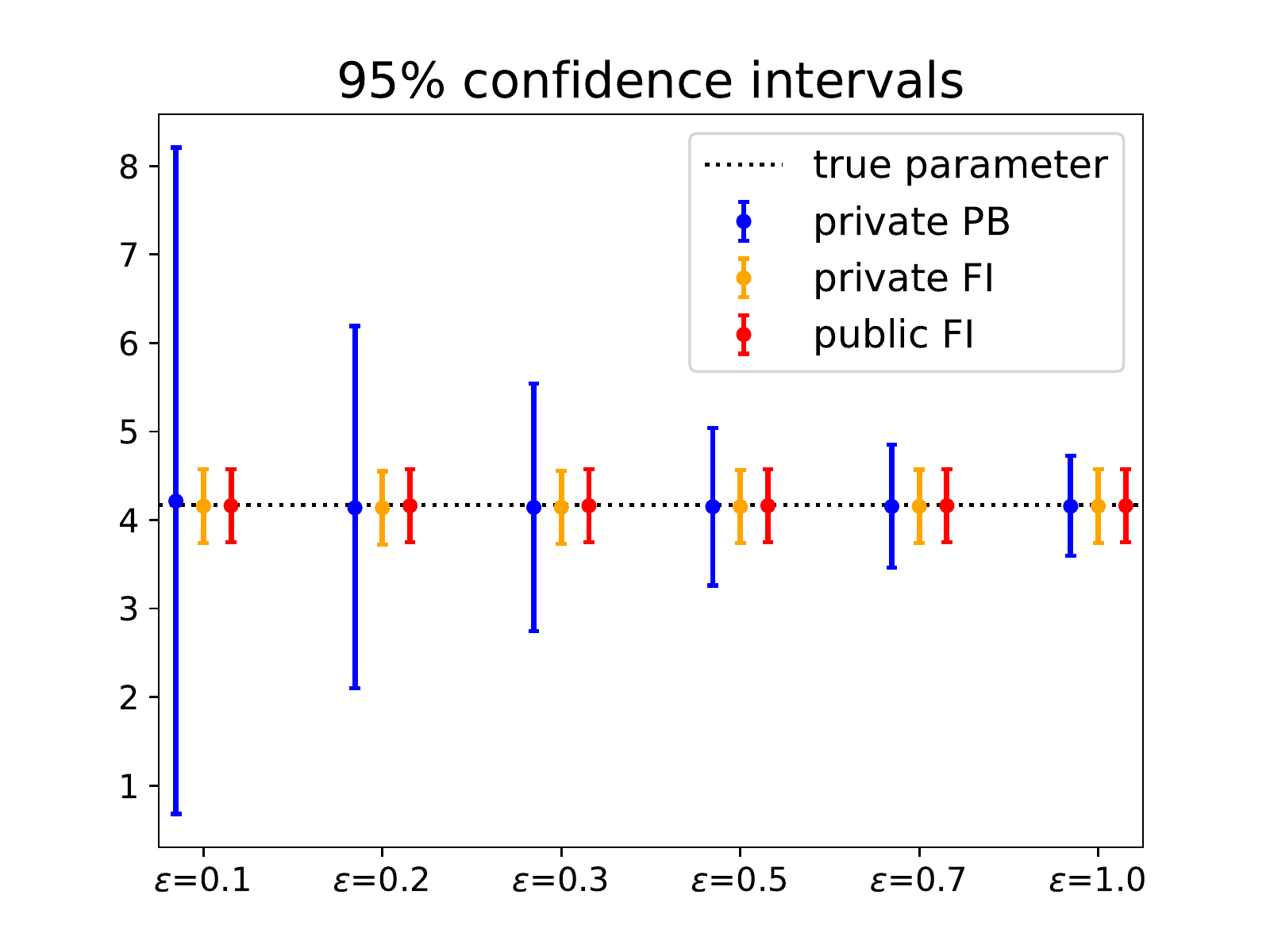}
  \caption{}
\end{subfigure}
\caption{Effects of varying $\epsilon$ for a fixed $n=100$. We selected a Gamma with inference on the scale parameter. The results are qualitatively equivalent for other distributions. (a) Observed coverage vs. nominal coverage of CIs. Coverage levels: $\{0.5, 0.6, 0.7, 0.8, 0.9, 0.95, 0.99\}$. From top to bottom: (i) differentially private parametric bootstrap; (ii) differentially private Fisher intervals; (iii) non-private Fisher CIs. Private methods use SSP via Laplace mechanism with varying values of $\epsilon$. Note that the effect of increasing $\epsilon$ with $n$ fixed is qualitatively similar to the effect of increasing $n$ holding $\epsilon$ fixed. (b) Average CIs for the scale parameter for different $\epsilon$. The width of the private bootstrap CIs approaches that of the public CIs as $\epsilon$ increases.}
\label{fig:eps}
\end{figure*}

\begin{figure*}[h!]
    \centering
    \includegraphics[scale=0.35]{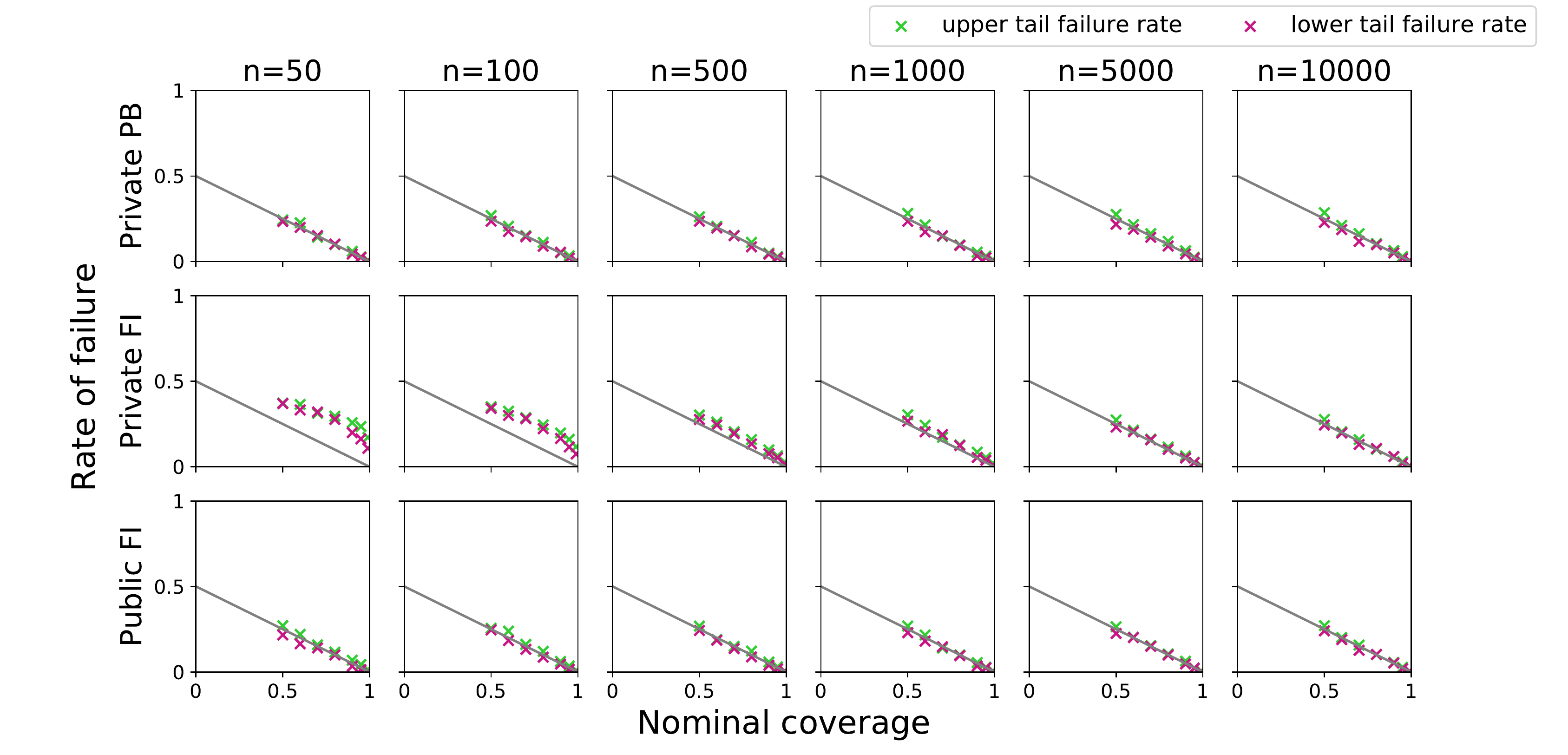}
    \includegraphics[scale=0.35]{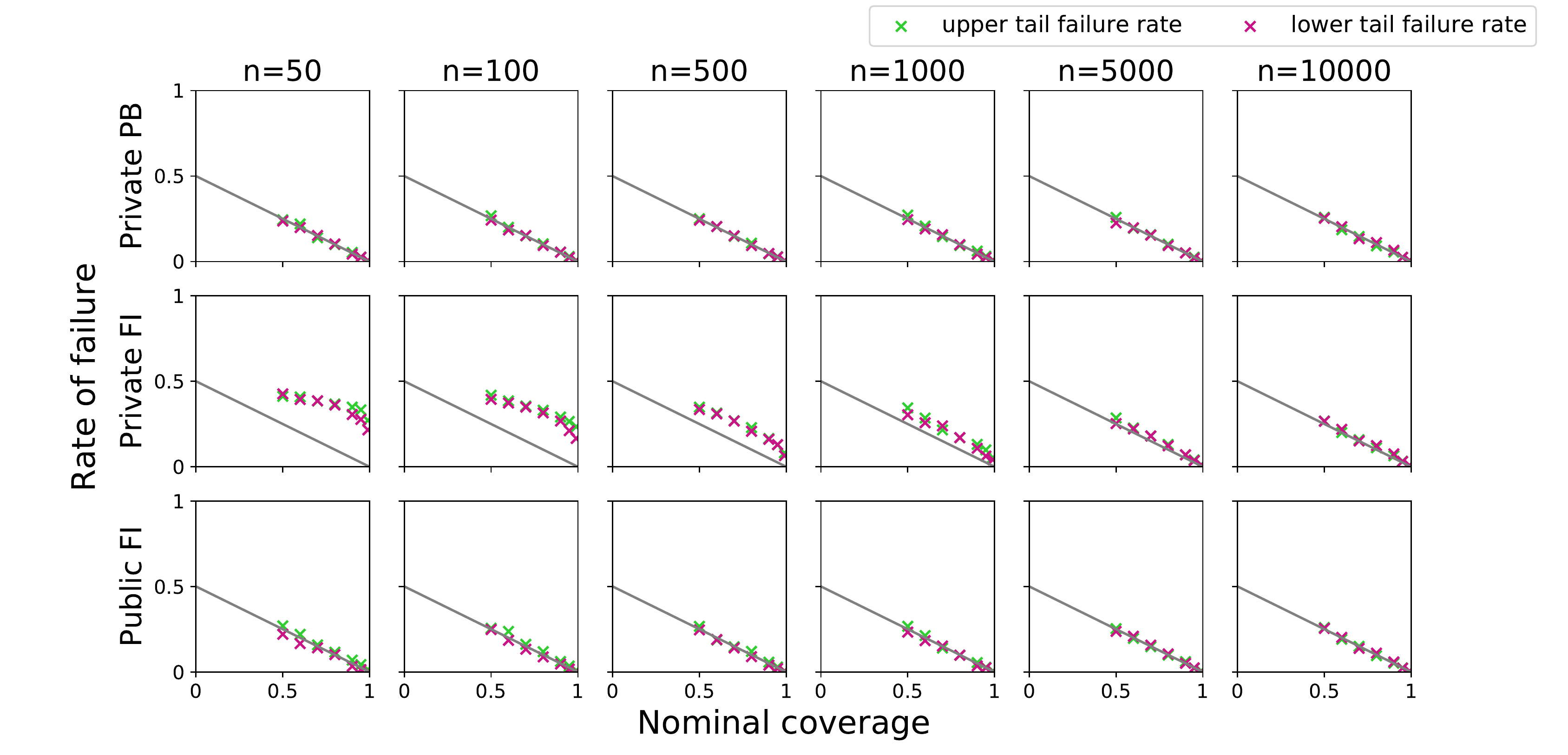}

\caption{In this Figure, we look at the rate of failure of the confidence intervals on the upper vs lower tail. For each of the two plots, the rows represent (i) differentially private parametric bootstrap; (ii) differentially private Fisher intervals; (iii) non-private Fisher intervals. Top: data range and sensitivity computed as described in Section \ref{sec:experiments}. Clamping the data to a range can introduce a bias if the range is not conservative enough. The bias becomes noticeable for large $n$, where the interval width is smaller. In our case, where the range is approximated from a data set of size 1000, a small bias becomes noticeable for $n\geq5000$, where upper-tail failures systematically outnumber lower-tail failures by a small margin. Bottom: same as top plot, with double the range. Increasing the range mitigates the bias.}
\label{fig:upperlower}
\end{figure*}



\begin{figure*}[h!]
\begin{subfigure}{.71\textwidth}
  \centering
 \includegraphics[width=1\linewidth]{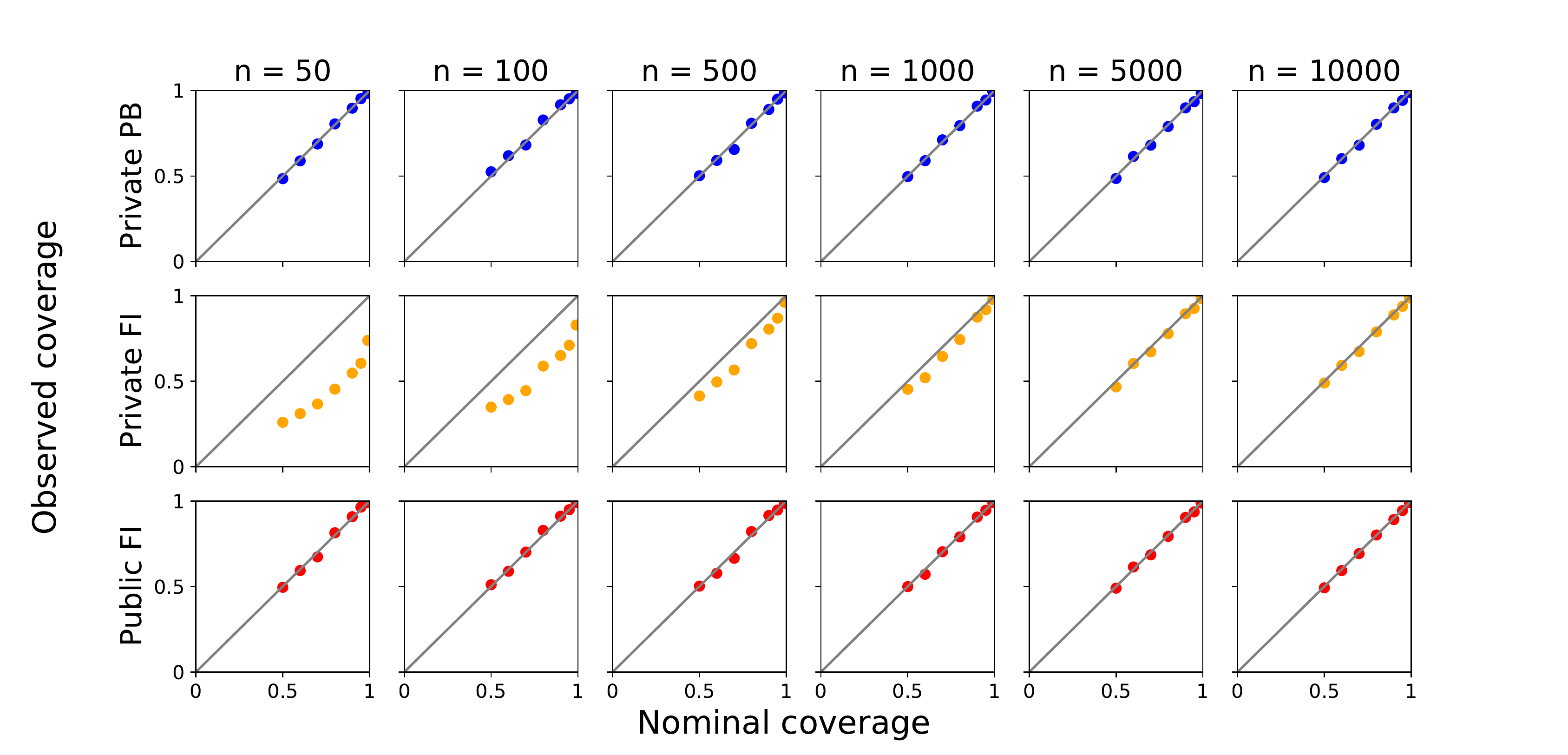}
  \caption{}
\end{subfigure}%
 \hspace*{-0.8cm}
\begin{subfigure}{.36\textwidth}
  \centering
  \includegraphics[width=1\linewidth]{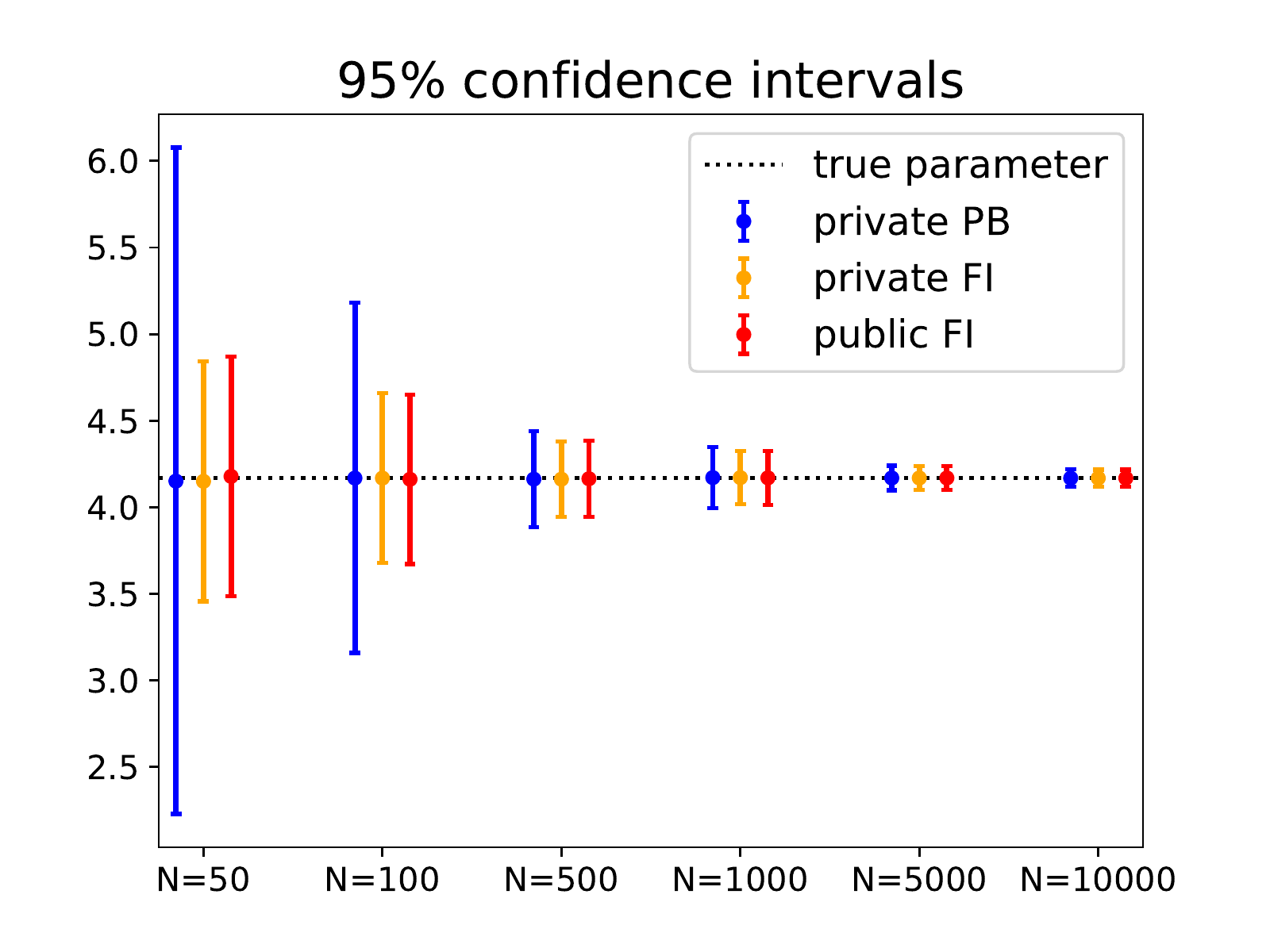}
  \caption{}
\end{subfigure}

\caption{Observed vs nominal coverage (left) and average CI width (right) for a multivariate Gaussian in 5 dimensions, with $\epsilon = 0.5$. We compute CIs for each dimension separately and report results for the first dimension as an example.}
\label{fig:OLSupperlowerols}
\end{figure*}

\begin{figure*}[t!]
\centering
 \includegraphics[scale=0.6]{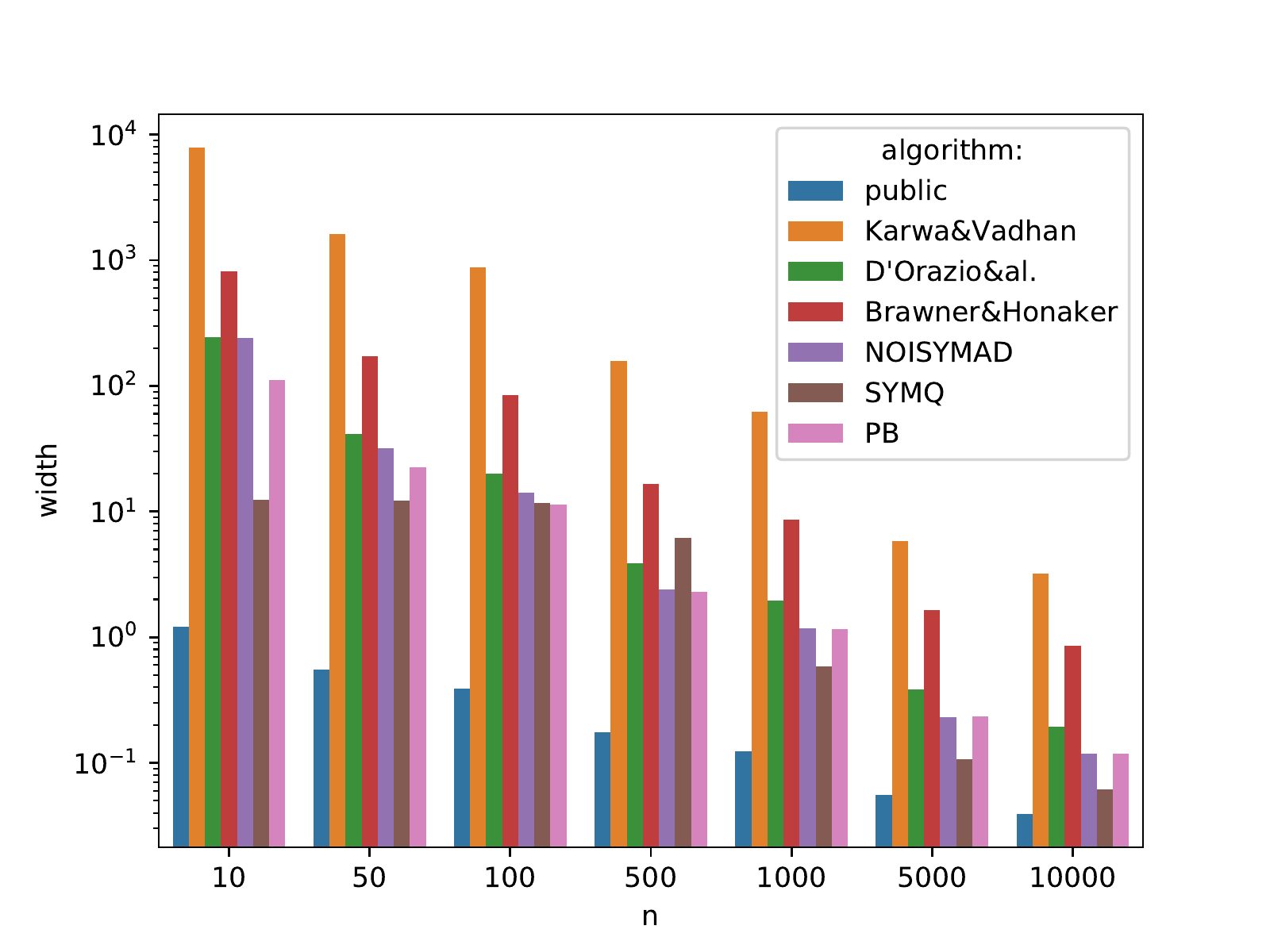}
\caption{For different algorithms, average width (logscale) of differentially private confidence intervals for the mean of a standard normal, range $[-8, 8]$, $\epsilon = 0.1$, for different $n$ levels. ``public'' is the confidence interval computed without differential privacy; ``Karwa\&Vadhan'' refers to~\cite{karwa2017finite}; ``D'Orazio\&al.'' refers to~\cite{d2015differential}; ``Brawner\&Honaker'' refers to~\cite{brawner2018bootstrap}; ``\texttt{NOISYMAD}" and ``\texttt{SYMQ}'' are methods from~\cite{du2020differentially}, and in particular ``\texttt{NOISYMAD}'' is very similar to our parametric bootstrap method (``\texttt{PB}''). We used the publicly available implementation by~\cite{du2020differentially} to reproduce their methods as well as the other prior methods.}
\label{fig:compare}
\end{figure*}

\end{document}